%% file: TNNLSRevised.tex
\documentclass[journal,12pt,onecolumn,draftclsnofoot]{IEEEtran}
\usepackage{times}
\usepackage{amsmath,epsfig}
\usepackage{amssymb}
\usepackage{amsfonts}
\usepackage{array,dcolumn}

\usepackage{algorithm, algorithmic}
\usepackage{subcaption}
\usepackage{psfrag}
\usepackage{balance}
\usepackage{cite}
\usepackage[all]{xy}
\usepackage{dsfont}
\usepackage{hyperref}
\usepackage{url}
\usepackage{xcolor}
\usepackage[nolist]{acronym}
\usepackage{float}
\usepackage{threeparttable}
\usepackage{cancel}
\usepackage[official]{eurosym}
\usepackage{graphicx}
\usepackage[utf8]{inputenc}
\usepackage[english]{babel}
\usepackage{amsthm}
\usepackage[subnum]{cases}
\usepackage[nocomma]{optidef}
\usepackage[normalem]{ulem}
\usepackage{enumitem}

\usepackage[utf8]{inputenc}
\usepackage{pgfplots}
\usepackage{booktabs, adjustbox}
\DeclareUnicodeCharacter{2212}{−}
\usepgfplotslibrary{groupplots,dateplot}
\usetikzlibrary{patterns,shapes.arrows}
\pgfplotsset{compat=newest}
\usetikzlibrary{shapes.geometric, arrows}

\usepackage{multirow}

\usepackage{filecontents, pgffor}
\usepackage{listings, xcolor}

\usepackage{xcolor, colortbl}
\usepackage{mathtools,amssymb,lipsum, nccmath}
\usepackage{cuted}
\usepackage[normalem]{ulem}

\newtheorem{definition}{Definition}
\newtheorem{proposition}{Proposition}
\newtheorem{Remark}{Remark}
\newtheorem{lemma}{Lemma}
\newtheorem{theorem}{Theorem}

\usepackage{tikz}
\usepackage{pgfplots}
\pgfplotsset{compat=1.10}
\usepackage{gincltex}
\usepgfplotslibrary{fillbetween}
\usetikzlibrary{patterns}

\def\BibTeX{{\rm B\kern-.05em{\sc i\kern-.025em b}\kern-.08em
		T\kern-.1667em\lower.7ex\hbox{E}\kern-.125emX}}

\hypersetup{
	colorlinks = true,
	linkcolor = blue,
	anchorcolor = black,
	citecolor = red,
	filecolor = blue,
	urlcolor = blue
}
\begin{document}
\title{Privacy-Aware Data Acquisition under \\ Data Similarity in Regression Markets}
\markboth{IEEE Journal Submission}%
{}
\author{Shashi Raj Pandey,%
		\thanks{Shashi Raj Pandey and Petar Popovski are with the Connectivity Section, Department of Electronic Systems, Aalborg University, Denmark. Email: \{srp, petarp\}@es.aau.dk.}
		\textit{IEEE Member}, %
		Pierre Pinson,%
		\thanks{Pierre Pinson has primary affiliation with Dyson School of Design Engineering, Imperial College London, UK. Email: \{p.pinson\}@imperial.ac.uk. He is also affiliated to the Technical University of Denmark, Department of Technology, Management and Economics, as well as with Halfspace.} \textit{IEEE Fellow}, %
		and Petar Popovski, \textit{IEEE Fellow}
		\thanks{This work was supported by the Villum Investigator Grant “WATER” from the Velux Foundation, Denmark.}}
 
\maketitle 
\begin{abstract}
Data markets facilitate decentralized data exchange for applications such as prediction, learning, or inference. The design of these markets is challenged by varying privacy preferences as well as data similarity among data owners. Related works have often overlooked how data similarity impacts pricing and data value through statistical information leakage. We demonstrate that data similarity and privacy preferences are integral to market design and propose a query-response protocol using local differential privacy for a two-party data acquisition mechanism. In our regression data market model, we analyze strategic interactions between privacy-aware owners and the learner as a Stackelberg game over the asked price and privacy factor. Finally, we numerically evaluate how data similarity affects market participation and traded data value.
\end{abstract}
\IEEEpeerreviewmaketitle

\section{Introduction}
\subsection{Context and Motivation}
In recent years, there has been a surge in Internet of Things (IoT) devices with sensing and computing capabilities, leading to an abundance of IoT data. Massively distributed heterogeneous data fuels various emerging applications, such as forecasting and analytics \cite{pinson2022regression}, learning models \cite{mcmahan2017communication,pandey2020crowdsourcing}, and in diverse industry verticals \cite{popovski2021internet, saputra2019energy}. It is thus critical to investigate the ways in which this data becomes available, respecting the privacy constraints and/or offering incentives to the data owners. \emph{Data markets} act as platforms that facilitate the collection, exchange, and utilization of both personal and IoT data. In this study, the focus is on \emph{regression data markets} \cite{pinson2022regression, dekel2010incentive}, which addresses the regression problem in which data is distributed across multiple agents. We motivate the use and operation of regression data markets through two examples. 

\textbf{Example 1}: Consider a labour market where agents Alice and Bob are generating privacy-sensitive distinct explanatory features $\mathbf{x}$ that may explain a common target variable $y$. Say, Alice wants to learn a regression model that quantifies the target value $y$, e.g., \emph{hiring salary} using features $\mathbf{x}$ (such as age, gender, academic qualifications, etc.). This is effectively done by learning the regression mapping function that involves features obtained from Bob as input. However, Bob's disclosure of features in a truthful manner leads to privacy loss and requires reasonable monetary compensation from Alice. Another agent, Carol, enters the market with an extensive history of employee hiring experiences, gathering similar features held by Bob to explain $y$. However, she is constrained by her company's intellectual property (IP) regulations, preventing \emph{direct} (unaltered) sharing of features with Alice and instead employing privacy guarantees and appropriate compensation. Bob's features are correlated with Carol's, and his willingness to share affects Carol's decisions, while Bob and Carol have different privacy preferences. 

\textbf{Example 2}: Consider a startup company offering a location-based ride-sharing service. There exist publicly accessible cameras at bus stations to capture passenger activities, owned by online platforms such as ``NAVER" \cite{naver} and ``Kakao" \cite{kakao}. Arrangements are made to uphold passenger privacy and confidentiality. The data derived from these cameras are of poor quality and are intentionally non-real-time due to security considerations. The startup wants to attain accuracy in predicting the volume of ride-sharing requests. Prediction accuracy hinges upon diverse passenger features, such as age, gender, number of passengers, and more. The setting also includes security companies and owners of cameras that monitor both these bus stops and nearby taxi stand. The startup can incentivize camera owners to retrieve data from these sources and enhance their predictions. However, the data streams from these disparate cameras, although resembling each other, necessitate measures to protect passenger privacy, constraining the attainable prediction accuracy. 

A regression data market allows data-driven analysis using a conditional model that explores the dependency between the feature vector and the target value of interest. In Example 1, Alice acts as a \emph{learner} and can initiate a collaboration protocol to build a regression model, soliciting features from Bob and Carol. Alice can employ any regression mapping functions to the distributed features, subsequently assessing the collaboration's merit using a chosen convex loss function, such as mean square error (MSE). Both Bob and Carol share heterogeneous privacy preferences and offer data of different qualities, such that Alice must incentivize both of them, employing tailored pricing signals. In Example 2, companies owning different cameras observe similar data and are market competitors. These companies may agree to provide their data to the startup in a manner that is equitable and ensures data privacy. Their aim is to limit the exposure of data similarities and the ensuing ramifications, such as data value depression \cite{acemoglu2022too, pandey2023strategic}, while simultaneously meeting passenger privacy mandates. However, privacy breaches can arise from collaboratively computed target values or direct data exchanges among agents. This includes questions about who controls IoT devices and maintains control over the collected data as well as whether strategic agents can even exert influence over the outcome of computations.

In both examples, the agents are sensitive towards information leakage due to data similarity.
However, it is important to note that, in a practical data market, the agents might be unaware of such data similarities, particularly due to privacy constraints and the added cost of sharing data to assess correlation. This indicates the agent's desire for disclosure of private information is often independent of potential data similarities between them unless there exist other related implications of data similarities, such as on the value of shared data and corresponding pricing of it, which they want to mitigate, as discussed in \cite{acemoglu2022too, pandey2023strategic}.  In such cases, consequently, they might opt for strategies like secure data computation methods (e.g., MPC \cite{cramer2015secure, mpc}) or leverage differential privacy (DP) techniques \cite{dwork2008differential, fallah2022optimal} to execute statistical analysis while preserving privacy. Nevertheless, employing privacy mechanisms, such as DP, invariably leads to a trade-off between data privacy and the accuracy of learned statistical models. The application of rigorous data privacy techniques during data exchange can potentially compromise the model's efficacy, with diminishing accuracy, particularly in situations characterized by increased data similarity.

Thus, the learner is challenged with optimising query signals to extract distributed features. Two critical aspects should be considered: (i) the value brought forth by each agent's features in addressing the regression problem and (ii) the impact of correlated features. These considerations are bounded by constraints pertaining to pricing budgets and the inherent diversity in agent privacy preferences. This brings us to the main question addressed by the present study: 
\emph{``How to find the optimal trade-off between data privacy and agent's utility under data similarity in a regression data market?''}
\subsection{Related Work}

Existing literature \cite{agarwal2019marketplace, acemoglu2022too, pinson2022regression, dekel2010incentive} on data markets explore various interactions among agents, formulating algorithmic solutions. The interplay between data privacy and ownership significantly influences the broader data trading process and the quality of offered services provided through data exchanged. This duality, often termed the \emph{privacy-utility} trade-off \cite{dwork2006calibrating, dwork2008differential}, can be addressed through incentives that balance data privacy considerations without compromising utility.  In the context of a regression data market, an effective incentive structure becomes pivotal in aligning the strategies of distributed data sources towards a shared goal, like training a learning model \cite{mcmahan2017communication, konevcny2016federated, pandey2020crowdsourcing}. The overall operation is challenged by the data characteristics, competition \cite{acemoglu2022too, pandey2023strategic}, privacy requirements and the computing capabilities of agents.

Recent works on IoT data markets focus on data acquisition models leading to efficient incentive mechanism designs. The objective is to improve participation for various applications, such as training a learning model \cite{pandey2020crowdsourcing, nguyen2021marketplace} and data utility maximization. Auction-based designs are studied in the data market design, leading to truthful mechanisms for data exchanges. Numerous privacy-preserving distributed model training approach exists where updates on local raw data are shared rather than the data itself, as done in Federated Learning (FL)\cite{mcmahan2017communication, konevcny2016federated}. Alternatively, variants of DP methods are implemented \cite{dwork2008differential, dwork2014algorithmic}, where the agents supply features they have for training at a central entity after adding noise proportional to privacy sensitivity. Therein, it is required to quantify the value of an individual agent's contribution in lowering the prediction error of the trained regression model. However, strategic interaction between agents having heterogeneity in local privacy preferences and developing a participation mechanism that respects the impact of the unknown degree of information leakage due to data similarity (as hinted in Example 2) has not been explored yet in a regression market setting. In \cite{pandey2023strategic}, the authors proposed a game-theoretic setup that enables distributed coalition amongst devices with similar data properties to minimize information leakage to prevent adversarial under-pricing and data rivalry issues in the data market. However, they ignore the influence of heterogeneity in individual privacy preferences. In \cite{falconer2023bayesian}, the authors introduced a \emph{Bayesian regression markets} design, which is compatible when considering a more general class of regression tasks. The work builds the mechanism by adopting a Bayesian framework that offers a fair allocation of market revenue for data trading. 
Regarding the literature gap, the existing works  \cite{pandey2023strategic, pinson2022regression} mostly focus on incentive design, participation, and the value of data trading. Here, instead, we look into the optimal participation strategy of privacy-sensitive data holders to realize a regression data market under information leakage due to data similarity. 

\subsection{Contributions and Paper Structure}

Unlike previous works, mainly designing algorithms of data acquisition frameworks for high-quality training regression models, both \emph{batch} and \emph{online}, as an initial attempt, our work focuses on a different and novel perspective: incentive mechanism design in regression market to limit information leakage due to data similarity amongst participating agents. In particular, our work outlines a framework to characterize practical scenarios where privacy preferences must be considered when evaluating the impact of information leakage due to data similarity on the participation of agents in training the regression models. We model a data acquisition method where we jointly analyze the impact of data similarity, particularly correlation, on statistical information leakage in a privacy-aware finite player regression data market. We develop incentive strategies that enforce loss minimization objectives during collaborative data trading to solve the regression problem. This is done by taming agents to provide high-quality data through incentives. In other words, this translates into finding the right purchase of a minimum amount of privacy from individual agents to train a high-quality regression model. The proposed incentive design offers control of the heterogeneous privacy factors on the distributed data and elicits them to the \emph{learner}. For the offered pricing and privacy budget, we summarize the strategic interaction between the learner and the agents following a single-leader, multiple-followers Stackelberg game structure, with the learner having a first-mover advantage and acting as a leader, while the devices are the followers. We show the Nash best response strategies of the agents lead to a unique equilibrium in the non-cooperative game amongst followers. With devices employing the local DP technique to trade data, we show there exists a trade-off between the available data privacy budget and the value of device participation. Finally, we show extensive numerical evaluations to verify these observations.

The paper is structured as follows. Section~\ref{sec:sys_model} introduces the system model and problem setup. Section~\ref{sec:framework} develops the interaction framework between agents in the regression data market as a two-stage Stackelberg game. This involves designing utility models and the mechanism design with a first-order, low-complexity iterative solution to the posed problem. Section~\ref{sec:simulations} provides a performance evaluation of the proposed algorithm with extensive numerical results and shows comparative analysis with the intuitive baselines. Finally, Section~\ref{sec:conclusion} concludes this work. 

\section{System Model and Problem Setup} \label{sec:sys_model}
We consider a network comprised of agents accumulating distinct data features that, in principle, contribute to learning the parameters of a regression model \cite{pinson2022regression}. These agents can communicate with each other through a platform that hosts a data market for training a regression model: a \emph{regression data market}. Then, any one of the interested agents can position herself as the ``learner'' and initiate the exchange of the feature observations. In particular, the learner can send a query request to incorporate high-quality features from other agents to improve the overall prediction accuracy of her model. Meanwhile, as agents are privacy-aware, they execute differential privacy methods (details on DP in Definition~1, Section~\ref{sec:framework}-A) to avoid potential privacy loss during data trading with the learner. Therefore, to maintain the quality of the collected features and the utility they bring, the learner specifically asks agents to abide by a common privacy measure (i.e., the \emph{privacy factor}) in sharing their features for an offered pricing. However, in practice, these devices exhibit heterogeneity in their data privacy preference, \emph{unknown} to the learner, as opposed to the commonly asked privacy factor. This suggests the strategic interaction of agents with the learner. Further, the data features they hold might be correlated, but to an \emph{unknown} degree, resulting in possible information leakage during features trading. In principle, correlated data contribute less to improving the performance of the trained model. Therefore, not all agents but a subset of available agents in the platform participate strategically in training the regression model by supplying high-quality features. In the following, we formalize the problem setup.

\textbf{Agents}. Like~\cite{pinson2022regression}, we consider a regression market consisting of a set of $n$ agents $\mathcal{A}=\{a_1,a_2,\ldots,a_n\}$. One of these agents $i$, called \emph{learner} $a_i$, is trying to fit a regression model based on its own features as well as features bought from other agents. The learner is initiating interaction with the other agents through the platform. Following the naming convention in~\cite{pinson2022regression}, we call the learner $a_{i}$ the \emph{central agent} and the remaining agents $a_j, j \ne i$ the \emph{support agents}. The term ``data samples" means ``features" in the context of regression tasks and will be used interchangeably throughout this work. Without loss of generality, the learner is $a_1$.

We assume that data samples are collected by agents at discrete points in time $t \in \{1,2,\ldots, \tau\}$. The central agent has a response target variable $Y_t\in\mathcal{Y}$, denoted for each time instance, and is trying to train a regression model that allows it to understand some statistics of $\{Y_t\}$, collected over $\tau$ time instances. Then, a time series collection $\{y_t\}$ is obtained, which denotes the realization of target variable $\{Y_t\}$, one for each time instance. The regression model relies on several explanatory variables (commonly called input features) indexed by a set $\Omega$. Consider $K$ explanatory variables such that $\Omega = \{x_k, k = 1,\hdots, K\}$ and denote $x_{k,t} \in\mathbb{R}, \forall k,t,$ be the $x_k$ feature observed\footnote{For simplicity, we restrict $x_{k,t} \in\mathbb{R}$.} at time instance $t$. Correspondingly, $\mathbf{y}=[y_1,\ldots,y_\tau]^\top$ is the collected target variable. Following our assumption that the features are observed at each time instance, we denote $\mathbf{x}_k = [x_{k,1}, \hdots, x_{k,\tau}]^\top$ be the vector of all observed values for the feature $x_k$ and $\mathbf{x}_t = [x_{1,t}, \hdots, x_{K,t}]^\top$ be the vector of all features observed at time instance $t$. Each agent in $\mathcal{A}$ holds a subset of relevant explanatory variables $\omega \subset \Omega$ and $\mathbf{x}_{\omega,t}$ the vector of features at time instance $t$. We assume that the data is complete and non-redundant, i.e., for each explanatory variable at each time step, a unique agent is holding the corresponding feature. However, we will not assume these explanatory variables are uncorrelated with each other, i.e., there is a possibility of having data similarity between samples of different agents. If we closely look at this setting, it is a setup in the vertical model training setting of federated optimization \cite{yang2019federated}, where the features are distributed across agents. This requires the acquisition of data distributed amongst several privacy-aware agents to identify the explanatory variable over time, which is our focus in this work. We denote $X_\omega \in \mathbb{R}^{\tau \times |\omega|}$ as the design matrix, whose column denotes the features observed, and the $t$'th row is $\mathbf{x}^\top_{\omega,t}$.

\textbf{Learning Model.} The central agent $a_i$ holds a set of features $\omega_i \subset \Omega$ and the target variable $y$. In a general setting, each agent $j \in \mathcal{A}\setminus a_i$ has a set of features $\omega_j \subset \Omega$, such that $|\omega_i| + \sum_j|\omega_j| = K$. As the features of potential relevance are distributed amongst the supporting agents $j$, at a regular time interval or at every time instance, the central agent aims to obtain those features to maximize its prediction ability on the target variable. Then, the regression problem of the central agent is to describe a mapping function $f$ between the set of explanatory variables $\omega \subset \Omega$ and the target variable $y$, i.e., $f: \mathbf{x}_{\omega,t} \in \mathbb{R}^{|\omega|} \rightarrow y_t \in \mathbb{R}$. As in the usual regression setup, the structure of $f$ defines the set of explanatory variables. For simplicity, we consider a linear regression mapping that can be described fully with the set of parameters $\boldsymbol{\theta}_\omega = [\theta_0,\theta_1,\ldots,\theta_d]^\top$, where $d= |\omega|$. Then, the mapping can be described as $y_t = \theta_0 + \sum_{k|x_k \in \omega}\theta_kx_{k,t} + n_t, \forall t$, with $n_t$ defining a Gaussian noise with zero mean and unit variance. The learning objective is to find the optimal set of parameters $\hat{\boldsymbol{\theta}}$ that minimizes the chosen loss function $l$, commonly taken as a quadratic function of prediction errors (residuals) denoted $e_t = y_t - \boldsymbol{\theta}_\omega^\top\mathbf{\tilde{x}}_t$, in expectation, as
\begin{equation}
\hat{\boldsymbol{\theta}}_\omega= \arg\min_{\boldsymbol{\theta}_\omega}\mathbb{E}[l(e_t)],
\end{equation}
where $\mathbf{\tilde{x}}_t = [1, \mathbf{x}_t]^\top$ with the first element a unit value to incorporate the bias term $\theta_0$ during computation. Here, $\boldsymbol{\theta}_\omega^\top\mathbf{\tilde{x}}_t$ results in the prediction $\tilde{y}_t$ with the parameter $\boldsymbol{\theta}_\omega^\top$.
In some settings, the central agent might have additional budget constraints, and the optimization of $\ell$ will happen subject to these constraints in this case\footnote{As an example, consider the central agent incentivize supporting agents to find the best estimate of $\hat{\boldsymbol{\theta}}_\omega$. Then, the budget constraint at the central agent can be defined in terms of the total available monetary value to offer.}.

The map $f$ is commonly described in a matrix form $\tilde{\mathbf{y}}=\tilde{X}_\omega\boldsymbol{\theta}_\omega$ with the $t$'th row of $\tilde{X}_\omega$ is $\tilde{\mathbf{x}}_{\omega,t}$. The model is then completely determined by the design matrix $\tilde{X}_\omega$, where $\tilde{X}_\omega$ is constructed executing the data acquisition protocol. For example, the loss function at any time instance can be given by
\begin{equation}
  \ell(\mathbf{y};\boldsymbol{\theta}_\omega)=\lVert \mathbf{y} - \tilde{X}_\omega\boldsymbol{\theta_\omega} \rVert^2,
  \label{eq:model}
\end{equation}
where the central agent aims to obtain the best parameter $\boldsymbol{\theta}_\omega$ that solves (1).
In the considered learning setup, we define $\tilde{\boldsymbol{\theta}}_{\omega, t}$ as the optimal set of parameters minimizing the expected loss $l$ over time and  $\tilde{L}_{\omega, t}$ as the time-varying estimator of the loss function. 
Note that, in an online learning setting, the optimization problem for updating the parametric information follows a recursive approach, where new information is a function of the latest residuals (c.f. \cite{pinson2022regression}, Eq. (16) and Eq. (17)). To that end, we consider $q_n \in [0,1], \forall n \in \mathcal{A}\setminus a_{i}$ is the participation probability to define the involvement of privacy-aware supporting agents in trading explanatory data samples. This means the construction of $X_\omega$ at the learner relies on the randomly perturbed features (see Definition 1, Sec.~\ref{subsec:strategies}) from the supporting agents, particularly as per their individual data privacy preference. Intuitively, as in \cite{zheng2020privacy}, such noisy, lower-quality contributing features impact the learner's ability to correctly map input features to the target variable, following larger parameter estimation errors. Furthermore, with $q_n = 0$, the agent $n$ opts out of data trading in the market.

\textbf{Problem Definition}. At any time instance, the ultimate goal of the central agent is to find the optimal mapping parameters \footnote{In principle, the optimal mapping parameters can be static and time-varying as per the regression problem setting. In any case, this will not influence the overall analysis made hereafter.} $\hat{\boldsymbol{\theta}}$ executing the data acquisition protocol with the supporting agents. The central agent aims to maximize the regression model performance by influencing supporting agents with appropriate pricing signals to exchange high-quality data, i.e., with asked perturbations as the privacy factor in DP terms. The supporting agents tune their responses per individual privacy preferences while considering the impact of data similarity and the announced pricing. In the following, we present the details of the interaction framework.

\section{Interaction Framework in the Regression Market}\label{sec:framework}
In this section, we show the strategic interaction between the central agent and the supporting agents in the linear regression market. We will model the two-stage game and derive strategies of the supporting agents for participation given the asked pricing and privacy factor under a potential data similarity situation from the central agent at equilibrium, respectively. After the mechanism design, we analyze the properties of the derived solution through a low-complexity, first-order backward induction method. 
\subsection{Strategies of agents}\label{subsec:strategies}
In this subsection, we explore the strategic behaviour of agents in the regression market. Specifically, we characterize the privacy preferences of individual agents and the impact of data similarity on their participation strategy.  

Each agent $a_n \in \mathcal{A}$ has a preference on data privacy $\epsilon_n$ \emph{unknown} to the central agent. While the value of privacy preferences differs amongst agents, we assume the agents shared information about its distribution. Let $F_{\varepsilon}$ be the cumulative distribution function (cdf) capturing the realizations $\epsilon_n$ from the random variable (RV) $\varepsilon \sim U[\epsilon_\textrm{l},\epsilon_\textrm{u}]$. A higher value of $\epsilon_n$ implies a lower sensitivity towards data privacy, while a lower value means the agents prefer injecting more noise on the traded data -- generating perturbed statistics on explanatory features -- to meet tighter privacy requirements. Intuitively, smaller $\epsilon \sim f_{\varepsilon}$ corresponds to lower information leakage in the data market, where $f_{\varepsilon}$ is the probability density function (pdf) of $\varepsilon$. Then, the agent can be called \textit{$\epsilon-$type} to characterize their privacy preference, and the corresponding participation strategy is referred to as $q_n(\epsilon_n) \in [0, 1]$. With this, we have the following definition. 
\begin{definition}
 The participation strategy $q_n(\epsilon_n)$ of an agent $a_n$ is defined over its privacy preference $\epsilon_n$ as the probability of joining the regression market such that  $q_n(\epsilon_n) \in [0, 1]$. 
\end{definition} We use a shorthand $q_n$ hereafter. Therein, we formalize the privacy preference with the following definition of local differential privacy.
\begin{definition}[$\epsilon$-Local Differential Privacy \cite{dwork2008differential}]\label{def:dp}
A mechanism $M(\cdot)$ satisfies $\epsilon-$local differential privacy ($\epsilon-LDP$) for $\epsilon\ge0$, if an only if, given any input data sample $x, x'\in Dom(M(\mathcal{X}))$, we have
\begin{equation}
 \forall y \in \textrm{Dom}(M): \mathds{P}[M(x)=y] \le e^{\epsilon}\mathds{P}[M(x')=y], 
\end{equation}
where $M(\mathcal{X})$ is a mapping to discrete values denoting the set of all possible outcomes of $M$.
\end{definition}

As outlined, the direct consequence of data similarity, particularly the correlation between traded data, is the undervaluation of data and price allocation mismatch amongst the agents - leading to market distrust and, eventually, a dropout scenario. To mitigate information disclosure, the agents employ a local DP strategy that equivalently limits their participation contribution; as such, the offered reward gets lowered. Hence, the supporting agents are reluctant to participate beyond their privacy budget, as defined by the realizations of $\varepsilon$. 

Because the exact privacy preference of each agent is private information, the disclosure of such information induces additional privacy costs for the supporting agents \cite{verrecchia1983discretionary}; hence, it is often \emph{unknown} to the central agent. Instead, the central agent plays around with various pricing signals and the asked privacy factors to align the strategies of supporting agents in improving the trained model performance. On the other hand, to counter potential information leakage, the supporting agents compete non-cooperatively and align strategies -- \emph{as per their type} -- of injecting structured noise into the trading data. We assume any rational agent joins the data market with the pricing compatible with their incurred costs, both in terms of data privacy and information leakage due to adjustments in the privacy budget; however, they opt-out from the market if the asked noise level is out of their individual privacy preference.  

\begin{Remark}
We define statistical information leakage as the ability of the central agent to decode the true type of the supporting agents precisely, manipulating the pricing and data valuation. The supporting agents opt out of the market if the information leakage exceeds their privacy preference.  
\end{Remark}

We observe the interaction between the central agent and the supporting agents, therefore, can be realized as a single leader multiple followers Stackelberg game, where supporting agents are stimulated by the central agent (the leader) to align their strategies on participation and the adjustment of privacy factor as DP noise during data trading in a non-cooperative manner. The central agent holds a first-mover advantage in asking with an uniform pricing offer for privacy factors \cite{liu2016interference}, while the followers \textit{compete over} the offered pricing to supply data samples implementing a privacy factor that maximizes their utility in participation under potential data similarity conditions. We consider a platform-centric mechanism \cite{feng2019joint, luong2016data}, where the price allocation between supporting agents operates as \emph{ex-post} rather than \emph{ex-ante}; thus, the offered pricing is shared amongst the agents per their contributions with the implemented privacy factor.

Next, we model the utility functions that characterize such interactions between the central and supporting agents. Based on the utility models, we develop the two-stage game and show the existence of Stackelberg equilibrium in such interactions.

\subsection{Utility Models}\label{subsec:utilitymodels}
We use $p$ to define the offered pricing by the central agent for the asked privacy factor $\mathbf{\epsilon}$. Then, for a given $p$, we model the valuation of the 
central agent as a monotonically increasing concave function of the privacy factor: $U(\epsilon) = {\ln[\alpha \epsilon p + 1]}^{-\beta}$, where $\alpha > 0, \beta \in (-1,0)$ are system parameters. Intuitively, the proposed valuation function captures the improvement in the central agent's utility when obtaining high-quality data samples from the supporting agents. The supporting agents exhibit heterogeneity in their privacy preferences, which is \textit{unknown} to the central agent. This implies the response of the supporting agents with the asked privacy factor and the offered pricing differs based on the solution to their utility maximization problem. Note that we follow a general structure in defining the valuation function and the agent's utility (later in \eqref{eq:ca_utility} and  \eqref{eq:sa_utility}, respectively, for the central agent and the supporting agents), which is a difference between the valuation and cost \cite{yu2017mobile, luong2016data, niyato2016market}. Therein, we can define the central agent's utility as follows.

\textbf{Central Agent's Utility:} 
Making the standard assumption of the concavity of the utility function \cite{gorman1968structure, singh1984price}, we propose the learner utility function\footnote{As an initial attempt to stimulate the participation of the supporting with no worse than $\epsilon$ privacy factor on available data samples, this utility function is the cleanest and most general form we can construct. There could be ways to define the valuation function. However, it eventually boils down to a similar interpretation of the utility functions as long as the concavity holds.} is based on the valuation $U(\epsilon)$ with the following definition.
\begin{definition} \label{def:learner_utility}
Considering the improvement in model prediction accuracy $ L(\zeta)= 1/(|\tilde{L}_{\omega_i} - \tilde{L}_{\Omega}|)$, using features obtained from supporting agents through participation, the central agent's utility is defined as  
\begin{equation} 
 \begin{aligned}
   S (p;\epsilon) 
   &= L(\zeta)\frac{1}{{\ln[\alpha \epsilon p + 1]}^\beta} - p \sum\nolimits_{n \in \mathcal{A}} \mathbb{1}_{\epsilon_n(q_n) > \epsilon}, \label{eq:ca_utility}
\end{aligned}   
\end{equation}
where $\zeta < \zeta_{\textrm{ref}}$ is the relative accuracy of the trained regression model for a reference requirement of $\zeta_{\textrm{ref}}$ in the regression market, $p$ is the offering pricing of the central agent to stimulate individual participation with no worse than $\epsilon$ privacy factor on the available data samples.   
\end{definition}
Following Definition~\ref{def:learner_utility}, given the participation of all agents, we have $p\ge\sum_{n}p_{a_n}$; this reflects the ex-post allocation of the offered price based on the supporting agent's contributions, as discussed earlier. And for a known participation that quantifies $L(\zeta)$, $S (p;\epsilon)$ is decreasing in $\alpha$ and $\beta$, and is concave with the privacy factor $\epsilon$. In Fig.~\ref{fig:CA_utility}, we illustrate the influence of the asked privacy factor on the valuation of the central agent for different system parameters. We observe a larger $\beta$ offers flexibility in data privacy factor with fair compromise on the utility. It is of interest for the central agent to solicit high-quality data with a relaxed privacy factor for the available pricing budget. Then, the central agent adopts an ex-ante differentiated pricing scheme to incentivize supporting agents in terms of the valuation of their data as a model contribution (c.f., Definition~\ref{def:SV}).
\input{figutilityCA}

\begin{Remark}
Participation of supporting agents with quality data in the data trading improves the performance of the loss estimates. The model contribution is then defined and evaluated as the marginal contribution of individual participation, where $L_n(\zeta)$ is calculated using the standard Shapley Value \cite{shapley1953value}, following Definition~\ref{def:SV}.
\end{Remark}
We assume the central agent updates the privacy factor during interactions with the supporting agents.

At each time instance, the central agent chooses the pricing signal $p$ that maximizes its utility defined as a composite function of the performance improvement, expressed in terms of relative accuracy $\zeta$, and the information leakage due to the injection of the statistical uncertainty (i.e., the noise) $\epsilon$.
\begin{definition}\label{def:SV}
The contribution of supporting agents $a_n \in \mathcal{A}$ in each iteration of interaction as
\begin{equation}
    L_n(\zeta) = \frac{1}{|\mathcal{A}|!}\sum\nolimits_{\pi \in\Pi(q_n)}[V(\mathcal{A}^{\pi}_{a_n} \cup \{a_n\}) - V(\mathcal{A}_{a_n}^{\pi})],
    \label{eq:shapTMC}
\end{equation}
where $V(\cdot)$ is the standard valuation of traded data \cite{pandey2022fedtoken} -- commonly known as performance score -- contributing to improving the model accuracy or lowering the loss function.
\end{definition}
We then have three particular interpretations following the central agent's utility function, given that all supporting agents trade their data as per their privacy budget, as follows.

\begin{enumerate}[label=C-\Roman*.]
    \item When $\epsilon=0$, we have $S(p;\epsilon)<0$, considering the definition of the utility function of the central agent $S(p;\epsilon)$. Furthermore, due to privacy restrictions, for $\epsilon=0$, we also have $L(\zeta) \approx 0$, i.e., no contribution of the supporting agents. 
    \item When $\exists n: \epsilon_n \sim f_{\varepsilon}$, and $\epsilon \sim \mathrm{U}[\epsilon_\textrm{l},\epsilon_\textrm{u}]$, we have $S(p;\epsilon) > 0$, leaving the central agent to solve the following optimization problem:
	\begin{maxi!}[2]  
		{p}                               
		{S(p;\epsilon)} {\label{opt:P}}{\textbf{P:}}
		\addConstraint{\sum\nolimits_{n}p_{a_n}\leq p,  \label{cons1:budget}}
		\addConstraint{p>0, \label{cons3:positive_pricing}}
		\addConstraint{L(\zeta) > \zeta_{\textrm{ref}}}.
	\end{maxi!}   
	Problem \textbf{P} is, therefore, an integral structure of the mechanism design problem where the central agent plays with its pricing signal for arbitrary privacy restriction on the supporting agents to ensure a level of performance $\zeta_{\textrm{ref}}$. 
  \item Following C-I, we have $S(p;\epsilon) \le 0$ when $\epsilon \notin [\epsilon_\textrm{l},\epsilon_\textrm{u}]$. Conversely, this leads to a similar scenario where $\epsilon=0$, i.e., no participation; hence, $S(p;\epsilon) = 0$.
\end{enumerate}

\textbf{Supporting Agent's Utility} Each agent $a_n$ of type $\epsilon_n$ responds strategically over the offered reward $p_{a_n}$ to minimize costs on data privacy, for the privacy budget $\epsilon$ asked by the central agent, and the information leakage due to data similarity. Denote general privacy cost function for employing $\epsilon-$differentially private mechanism as $c_n(C_n, \epsilon): \mathbb{R}_+ \rightarrow \mathbb{R}_+$, where $c_n(\cdot)$ models the cost of participation employing a privacy factor of $\epsilon$ with the marginal participation cost parameter $C_n$ for agent $a_n$. We use a linear-cost model to simplify our analysis to define the agent's privacy cost, i.e., $c_n(C_n, \epsilon) = C_n\epsilon$. Hereafter, consider the shorthand $c_n$ to indicate this parameterized privacy cost. As indicated before and validated experimentally (refer to Example 3), the supporting agents experience utility loss under information leakage. We define $\mathbb{E} [\mathcal{I}(q_n,\rho_n)$ that quantifies the information leakage following supporting agents' participation with similar data. As per the analysis made with data acquisition models in \cite{cummings2015truthful, cummings2023optimal}, the measure of information leakage increases with the number of high-quality data samples in the regression market. Taking this as a reference, we modelled $\mathbb{E} [\mathcal{I}(q_n,\rho_n)]$ as strictly increasing with the number of active agents in the regression market. Considering this, the individual data owner aims to minimize the information leakage during data trading and tune privacy factor $\epsilon_n$ over the offered pricing $p_{a_n}$ for maximal benefit of participation in the data market:
\begin{equation}
    u_n(\epsilon_n, p_{a_n}) = \gamma V_n({q_n}, p_{a_n}) - \epsilon_n \mathbb{E} [
    \mathcal{I}(q_n,\rho_n; \epsilon)] - \psi_n c_n, \label{eq:sa_utility}
\end{equation}
where $V_n({q_n}, p_{a_n})$ is the valuation\footnote{For ease, we model it as a linear function of asked privacy factor for the offered pricing.} of agent $n$ on participation for the offered pricing $p_n$, $\gamma > 0$ captures the participation preference of supporting agents, wherein a larger $\gamma$ implies a higher valuation on participation, and $\mathbb{E}[\mathcal{I}(q_n,\rho_n)]$ is a strictly increasing that captures information leakage due to data similarity with added privacy cost on participation $c_n$; as discussed, we model it with the number of active agents $N(q_n):= \sum\nolimits_{n\in\mathcal{A}}\mathbb{1}_{q_n=1}$ as $\varphi_n \log(1+N(q_n))$, and $\varphi_n$ and $\psi_n$ are weight parameters on information leakage and local privacy cost for agent $n$, respectively. We further exemplify information leakage with the following numerical example and discuss the utility model design.

\input{figleakageimpact}
\textbf{Example 3}: In Fig.~\ref{fig:contribution}, we provide an example scenario to assess the impact of information leakage due to data similarity on the valuation of the supporting agents' data. We set supporting agents $\mathcal{A}\setminus{\{a_{1}\}}=\{\{a_2\},\{a_3\},\{a_4\}\}$ aiding the linear regression market initiated by the central agent $\{a_{1}\}$. The numerical evaluation follows the settings of \cite{pinson2022regression}. In addition, we consider agents $\{a_3\}$ and $\{a_4\}$ have correlated data samples denoted as $\rho_{3,4}$ while agent $\{a_4\}$ is injecting noise $\sigma_4$ of different magnitude. The central agent is solving the regression problem, as defined in \eqref{eq:model}, while the normalized contribution of each supporting agent is evaluated following Definition~\ref{def:SV}. The heatmap reflects the impact on the agent's normalized contribution within this setup. We begin by obtaining the marginal contribution of agent $\{a_4\}$ higher than the agent $\{a_3\}$. Note that the obtained numerical results could differ per the underlying model considered; however, the discussions and intuitions following this analysis are still valid. Next, in the corresponding heatmap scales, we observe the information leakage due to data similarity between agents $\{a_3\}$ and $\{a_4\}$ impacts the normalized contribution of both agents. \textit{First}, when the agent $\{a_4\}$ injects more noise $\sigma_4$, this lowers its \emph{own} share of contribution in improving the model accuracy. Then, following the considered setup, the contribution of $\{a_4\}$ is lowered when $\sigma_4$ is high but is still better than that of the agent $\{a_3\}$. \emph{Second}, with increased data correlation $\rho_{3,4}$, we observe a similar trend in the normalized contributions for each agent, which is intuitive considering the impact of information leakage due to data similarity. Furthermore, having a high degree of noise injection in the shared data and perfect data correlation would lead to having agent $\{a_2\}$ as the only contributor in solving the regression problem.

Next, we formalize our proposed mechanism design that allows privacy-aware data acquisition under data similarity in the regression market.

\subsection{Mechanism Design}\label{susbsec:mechanism}
In this subsection, we derive the optimal strategies for the participation of the supporting agents through their sequential interaction with the central agent over pricing. We build on the utility models designed in Sec.~\ref{subsec:utilitymodels} and formulate the two-stage game model of interaction under data similarity.

Recall the utility models defined in \eqref{eq:ca_utility} and \eqref{eq:sa_utility}. While it is true that agent $n$ accrues a utility $u_n(\epsilon_n, p_{a_n})$ only for a positive pricing signal, i.e., $p_{a_n} > 0$, such that
            \begin{equation}
              u_n(\epsilon_n, p_{a_n})= 
                \begin{cases}
                u_n(\epsilon_n, p_{a_n}), & \text{if $\epsilon_n \in \varepsilon$} \\
                -\infty, & \text{otherwise}. \
                \end{cases}
            \end{equation}
We remark $p_{a_n}>0$ is a necessary but not sufficient condition for agent $n$ to participate in the data trading.

We have two participation scenarios to characterize the utility function of the individual agent, as follows:
\begin{enumerate}[label=S-\Roman*.]
    \item We have $q_n = 1$ when $\epsilon_n(q_n) > \epsilon$, where
    \begin{equation}
        \epsilon_n(q_n) = \frac{\gamma V_n({q_n}, p_{a_n}) - \psi_n c_n}{ \mathbb{E} [
    \mathcal{I}(q_n,\sigma_n; \epsilon)]},
    \label{eq:response}
    \end{equation} 
    and $q_n = 0$, otherwise.
    \item The requirements of $\epsilon_n(q_n) > \epsilon$ restrict some of the supporting agents from participation in the regression market, primarily due to individual privacy budgets.   
\end{enumerate}
Then, considering scenarios S-I and S-II, the following results are derivations of the optimal participation response of the supporting agents with necessary definitions.
\begin{definition}[\textbf{Feasibility}]\label{def:feasibility}
The mechanism is feasible for the offered pricing signal $p$ if $\exists n \in\mathcal{A}: u_n(\epsilon_n, p_{a_n}) > 0$. Feasibility criteria can be satisfied as the central agent is aware of the distribution on privacy preference profiles $F_\varepsilon$ and stimulates interaction for the exchange of data samples considering the agent with the highest privacy budget $\epsilon_u$ and set $\epsilon=\epsilon_u$. 
\end{definition}
Following Definition~\ref{def:feasibility}, and the utility profile of the supporting agents, however, we cannot guarantee that the criterion for joining the regression market will be fulfilled for all available agents $\epsilon_n(q_n) > \epsilon$, i.e., we have
\begin{equation}
    q_n = \int_0^\infty q(\epsilon) dF_\varepsilon(\epsilon)
\end{equation}
that quantifies the joining fraction of supporting agents, in probability, in the market. This leads to 
\begin{equation}
    q_n = \int_0^{\epsilon_n(q_n)} dF_\varepsilon(\epsilon) = F_\varepsilon(\epsilon_n(q_n)).
\end{equation}
Therefore, we can formalize the participation probability of the supporting agent as follows.
\begin{definition}\label{def:nash}
Given the following condition satisfies, as
\begin{equation}
    q_n^* = F_\varepsilon(\epsilon_n(q_n^*)),
\end{equation}
we define $q_n^*$ as a Nash equilibrium of the supporting agent.
\end{definition}

\begin{lemma}\label{lemma:equilibrium}
Given the incurred cost of data exchanges $c_n$ in the regression market, with a shared value of instantaneous information leakage, there exists a unique Nash equilibrium $q_n^*$ defining the probability of supporting agents joining the collaborative training in the regression market. 
\end{lemma}

\begin{proof}
We begin the proof for the uniqueness of the solution by defining a variable $\xi(q_n):= F_\varepsilon(\epsilon_n(q_n)) - q_n$. As $\mathbb{E}[\mathcal{I}(q_n,\rho_n)]$ is a strictly increasing in \eqref{eq:response}, consequently, we have $\epsilon_n(q_n)$ a strictly decreasing on its domain, and leading to $F_\varepsilon(\cdot)$ as an increasing function \cite{boyd2004convex}. Then, following the Definition~\ref{def:nash}, $\xi(q_n)$ should have a unique solution, i.e., a root, to guarantee $q^*_n$ is an equilibrium at the best response. Then, 
\begin{enumerate}
    \item if $V_n({q_n}, p_{a_n}) \le \bigg( \frac{\psi_n}{\gamma}\bigg)c_n$, we have $q_n^* = 0$, resulting in unique root of $\xi(q_n)$.  
    \item if $V_n({q_n}, p_{a_n}) \ge \frac{1}{\gamma}\bigg[\psi_nc_n + \epsilon_n\mathbb{E}[\mathcal{I}(q_n,\rho_n)\bigg]$ for any $\gamma > 0$, we have $q_n^* = 1$, resulting unique root of $\xi(q_n)$. 
    \item otherwise, if we have a region between (1) and (2), i.e., $c_n < V_n({q_n}, p_{a_n}) < \bigg(\frac{\gamma}{\psi_n}\bigg)\epsilon_n\mathbb{E}[\mathcal{I}(q_n,\rho_n)]$, there exists a unique root $q_n^* \in (q'_n, 1)$. This can be concluded based on the following observations. We also drop the normalizing constants hereafter for simplifying the analysis, as it won't influence the conclusion made. Choose arbitrary $q'_n \in (0,1)$, then there exists $c_n < V_n({q_n}, p_{a_n}) = c_n + \epsilon_n\mathbb{E}[\mathcal{I}(q_n,\rho_n)] < c_n + \epsilon_n \mathbb{E}[\mathcal{I}(1,\rho_n)]$ as $\mathbb{E}[\mathcal{I}(1,\rho_n)]$ is strictly increasing. Then, using the definition of $\xi(q_n)$, which is a continuous, decreasing function, we have $\xi(q_n) = 1 - q_n > 0, \forall q_n \in [0, q'_n]$ and $\xi(1) = F_\varepsilon(\epsilon_n(1))-1 < 0$.
\end{enumerate}
Next, we show that $q^*_n$ is the Nash equilibrium with the following observations. As such, the optimal strategy of supporting agents is to adopt their true privacy preference for incentives during data trading. From conditions in \eqref{eq:response}, this is straightforward as we have, 
\begin{enumerate}
    \item the expected utility $\tilde{V}_n$ of supporting agent $n$ is $\tilde{V}_n:= V_n({q^*_n}, p_{a_n}) - \frac{1}{\gamma}\bigg[\psi_nc_n + \epsilon_n \mathbb{E}[\mathcal{I}(q^*_n,\rho_n)\bigg] > 0$, if $q_n = 1$. For any other strategy $\tilde{q}_n \in [0,1)$, the expected utility is $\tilde{q}_n\bigg(V_n({q^*_n}, p_{a_n}) - \frac{1}{\gamma}\bigg[\psi_nc_n + \epsilon_n \mathbb{E}[\mathcal{I}(q^*_n,\rho_n)\bigg]\bigg) < \tilde{V}_n$.
    \item The converse is true, as deviating from optimal strategy $q_n = 0$ only lowers the expected utility of the supporting agent $\tilde{V}_n$, when $V_n({q^*_n}, p_{a_n}) < \frac{1}{\gamma}\bigg[\psi_nc_n + \epsilon_n \mathbb{E}[\mathcal{I}(q^*_n,\rho_n)\bigg]$.
    \item Finally, when $V_n({q^*_n},p_{a_n}) =\frac{1}{\gamma}\bigg[\psi_nc_n + \epsilon_n\mathbb{E}[\mathcal{I}(q^*_n,\rho_n)\bigg]$, any deviation in the supporting agent's strategy does not improve its expected utility, i.e., $u_n(\epsilon_n, p_{a_n})=0$; hence, the supporting agent has no incentive to deviate.
\end{enumerate}
This completes the proof.
\end{proof}

Following Lemma~\ref{lemma:equilibrium}, after meeting the participation criteria, the supporting agents aim to tune their participation strategies with a privacy factor that maximizes benefits, i.e., utility, over an offered pricing $p$ and an asked privacy factor of $\epsilon$ by the central agent. This means that the supporting agents derive the optimal response for maximizing their overall valuation with the pricing for the query made by the central agent by solving the following optimization problem.
	\begin{maxi!}[2]  
		{q_n(\epsilon)}                               
		{u_n(\epsilon, p_{a_n}| p )} {\label{opt:P1}}{\textbf{P1:}}		
		\addConstraint{V_n \ge \frac{1}{\gamma}\bigg[\psi_nc_n + \epsilon\mathbb{E}[\mathcal{I}(q_n,\rho_n)\bigg],\label{p1cons1:posutil}}	
		\addConstraint{\epsilon_n(q_n) > \epsilon \label{p1cons1:participation}},
		\addConstraint{p_n>0, \label{p2cons3:positive_pricing}}
	\end{maxi!}
where \eqref{p1cons1:posutil} ensures a positive return on participation and constraint \eqref{p1cons1:participation} satisfies the individual privacy budget. Consider $q_n(\tilde{\epsilon})$ is the solution of \textbf{P1}, then, we have $q_n^* = \min\{q_n(\tilde{\epsilon}), q^*_n\}$.

Following the solution to \textbf{P1}, i.e., the strategic participation, without loss of generality, we then formulate the overall regression market problem as follows:
	\begin{maxi!}[2]  
		{p, \epsilon}                               
		{S(p;\epsilon| \textbf{q}^*)} {\label{opt:P1}}{\textbf{P2:}}		\addConstraint{\sum\nolimits_{n}p_{a_n}\le p(\epsilon),  \label{p2cons1:budget}}	
		\addConstraint{\epsilon_n(q_n^*) > \epsilon > \epsilon_{\textrm{ref}}, \forall n \in \mathcal{A} \label{p2cons1:participation}},
		\addConstraint{p(\epsilon)>0, \label{p2cons3:positive_pricing}}
		\addConstraint{L(\zeta) > \zeta_{\textrm{ref}}},
	\end{maxi!}
where $\textbf{q}^*$ is a vector with the best response strategies of the supporting agents over the offered pricing signal and the privacy requirements set by the regression market; constraint \eqref{p2cons1:participation} restricts the asked privacy guarantees to a reference $\epsilon_{\textrm{ref}}$ value. While it is in the best interest of the central agent to keep the value of $\epsilon$ as large to improve participation with less perturbed explanatory data samples in the regression market, it is mostly impractical. This is due to the influence of heterogeneous privacy preferences. In principle, $ \epsilon_{\textrm{ref}}$ is set as $\max\{\epsilon_n\}, \forall n \in \mathcal{A}$ in each iteration of interaction to satisfy participation constraint $\eqref{p2cons1:participation}$, while satisfying the price allocation under budget constraints \eqref{p2cons1:budget}.  \begin{lemma}\label{lemma:optimal_response}
The optimal solution for \textbf{P2} $\epsilon^*$ for the known pricing budget at the central agent can be derived as $\max\{\epsilon_n(q_n^*), \epsilon_{\textrm{ref}}\},$ where the asked privacy guarantees are set by the central agent to a reference $\epsilon_{\textrm{ref}}$ value. 
\end{lemma}
\begin{proof}
The proof follows the characteristics of the constraints, leading the optimal solution to be the boundary conditions for a fixed offered pricing $p(\epsilon)>0$. Note that the differential price allocation constraint on the available monetary budget, $\sum\nolimits_{n}p_{a_n}\leq p(\epsilon)$ is evaluated as per the proportional contribution measure of the individual supporting agent, as in \cite{tun2019wireless, pandey2023strategic}. 
\end{proof}

\input{algo.tex}

For the pricing signal $p$ and the asked privacy budget $\epsilon$, following \textbf{P1} and problem \textbf{P2}, leads the regression market problem a two-staged leader-follower game, where the market aims to receive high-quality explanatory data samples for the feasible pricing signal to the agents who strategically response to the query made for participation in a non-cooperative setting. 

Following Lemma~\ref{lemma:optimal_response} results in the following overall utility maximization problem:  
	\begin{maxi!}[2]  
		{p>0, \mathbf{q}}                               
		{\frac{L(\zeta)}{{\ln[\alpha \epsilon^*p(\epsilon^*) + 1]}^\beta}- 
		\hspace{0.1cm} p(\epsilon^*) \sum\nolimits_{n \in \mathcal{A}}q_n \epsilon^*(q_n)} {\label{opt:P3}}{\textbf{P3:}}
		\addConstraint{L(\zeta) > \zeta_{\textrm{ref}}}
		\addConstraint{q_n \in \{0,1\}, \forall n \in \mathcal{A},}
	\end{maxi!} 
	
The solution to the optimization problem \textbf{P3} is non-trivial, first, given the participation constraints following response to feasibility (Definition~\ref{def:feasibility}), the influence of pricing on participation, and its overall consequence on the performance improvement factor $L(\zeta)$. Second, with a possible $2^{|\mathcal{A}|}$ configurations, it might require exponential-complexity effort to solve the problem with an exhaustive search solution. We propose a low-complexity solution to address this.

We develop a first-order iterative solution (Algorithm~\ref{alg:stackelberggame}) that aims to satisfy the conditions for the Nash solution and builds on top of the backward induction method to reach the Stackelberg equilibrium. The proposed method satisfies the following economic properties.

\emph{Incentive Compatibility}: The mechanism is incentive compatible given if all supporting agents behave rationally according to their local privacy preference, i.e., their true type, such that $\mathbb{E}[\epsilon_n, u_n(p_{a_n})] \ge \mathbb{E}[\epsilon, u_n(p_{a_n})], \forall \epsilon, \forall n \in \mathcal{A}$.

\emph{Individual Rationality}: For each supporting agent, there exists a non-negative utility $u_n(\epsilon_n, p_{a_n}) \ge 0, \forall n \in \mathcal{A}$, if they respond with their true type. This is the participation constraint, which is observed following Lemma~\ref{lemma:equilibrium}. The supporting agents always opt for their true data type, ensuring a positive utility.

Besides the fundamental properties of our mechanism, its computational properties are covered through the following theorem.
\begin{theorem} The first-order Algorithm~1 solves the overall utility maximization problem \textbf{P3} with linear complexity.
\end{theorem}
\begin{proof}
The proof follows the convexity property of the maximization problem in (line 6) with the responses of supporting agents to the asked privacy factor $\epsilon = \max\{\epsilon_n(q_n^*), \epsilon_{\textrm{ref}}\}$ (line 12), that eventually reduces the problem into a single variable optimization with the initial maximum offered pricing. 
\end{proof}
With this, we now have all the ingredients to characterize the \emph{Stackelberg equilibrium}. Obtaining the solution of \eqref{opt:P3} $p^*(\epsilon)$, involving best-responses of supporting agents with Definition~\ref{def:nash}, we have the following proposition.
\begin{proposition}
For any values of $p$ and $\epsilon$, we have the Stackelberg equilibrium if the following conditions are satisfied:
\begin{equation}
S(p^*, \boldsymbol{\epsilon^*} ) \ge S(p, \boldsymbol{\epsilon^*})
\end{equation}
\begin{equation}
 u_n(\epsilon_n(q^*_n), p^*) \ge u_n(\epsilon_n(q_n), p^*), \forall n \in \mathcal{A}.
\end{equation}
\end{proposition}

\section{Numerical Results}\label{sec:simulations}
This section introduces the evaluation results for the proposed framework. We begin with the evaluation setup, where we show the underlying model used in the regression task. Then, we provide a performance evaluation following an analysis of the pricing, strategic participation, and comparison under the impact of data similarity. 

\textbf{Setup:} We consider a plain regression learning problem. For this, we generate the data to match the setup in which four agents, where $\{a_{1}\}$ is a central agent posing the online regression problem, as in \cite{pinson2022regression}, and  $\{a_2, a_3, a_4\}$ are the supporting ones supplying contributing features. The agents use distinct features sampled from the Gaussian distribution with unit variance. Central agent $\{a_{1}\}$ own feature $x_1$, while the supporting agents $\{a_2, a_3, a_4\}$ hold relevant features $x_2$, $x_3$, and $x_4$, respectively, at each time step. We particularly consider a single-order regression model such as $y_t = \theta_0 + \theta_{1,t}x_{1,t} +\theta_{2,t}x_{2,t} + \theta_{3,t}x_{3,t} + \theta_{4,t}x_{4,t} + \beta_t$, where the last term $\beta_t$ is Gaussian noise with zero mean and a finite variance of $0.3$, with quadratic loss, as in \eqref{eq:model}. Furthermore, to demonstrate the impact of participation due to data similarity and, further, the information leakage, we model features as correlated with each other through linear models, as in \cite{pandey2023strategic}. We first simulate the process using the true parameters as $\mathbf{\theta}^\top= [0.2\; 0.4\; -0.3\; -0.6\; 0.2]$. We follow the batch estimation process [see Section 2.3.2 in \cite{pinson2022regression}, equation (15)], gathering features for $\tau = 10000$ times, and later use this initialization for the online regression. The privacy budget is set as $\epsilon_{\textrm{ref}} = \ln{10}$. and the agents strategically employ $\epsilon-\textrm{LDP}$ following Definition~\ref{def:dp}, while the utility model uses parameters $\alpha=0.45$ and $\beta = -0.4$.

\input{figcumpay1}
\textbf{Analysis on offered pricing and loss estimates:} For evaluation, we consider the following two intuitive baselines: (i) \textbf{Case I}, which ignores the strategic participation of the agent allocates pricing as per the contribution made by the agents, (ii) \textbf{Case II}, which considers strategic participation and allocates pricing as per the contribution made by the agents. The rationale behind our choice of baselines is as follows. This work is an initial attempt to understand the dependency of strategic participation, particularly with the impact of information leakage due to data similarity and heterogeneity in privacy preferences in a regression market. Therefore, our baselines are focused on establishing a comparative understanding of the validation and efficiency of our constructed model. Our method \textbf{Alg. 1} considers the strategic participation of devices with a proportional price allocation scheme as per their data contribution. In Fig.\ref{fig:cumpay}, we show the evolution of normalized payment to three agents with contributing features in the regression market established by agent $\{a_{1}\}$, where agent $\{a_{2}\}$ is sharing the poorest quality data (i.e., implementing extreme privacy measures) and agents $\{a_{1}\}, \{a_{2}\}$ have correlated data. We observe that in both cases, the agents are offered payments more than the asked price, with a variability of $5.6\%$ to $15.3\%$. This situation ensures the participation of the agents; however, they do not align their privacy budget accordingly as asked by the learner - leading to unintended consequences in the payment for the learner and its utility. Compared to Case I, in Case II, the payment is shared amongst agents as per their strategic participation, which influences individuation contributions.

In Fig.~\ref{fig:beta_online}, we see the temporal evolution of the model parameters as per the underlying model we have adopted. Correspondingly, Fig.~\ref{fig:loss_estimate_v1} validates the performance of the learned regression model for the online regression task with four agents. We consider three scenarios for the evaluation of loss estimates: (i) Central Info, where only the feature available at agent $\{a_{1}\}$ is provided; (ii) Partial Info, where only features of agents $\{a_{1}, a_2, a_3\}$ are solicited by the agent $\{a_{1}\}$; and (iii) Full Info, where all features are provided, i.e., full participation. It is intuitive that the temporal evolution of loss estimates with Partial Info is better than the Central Info, where the central agent is missing relevant features in keeping track of estimating the true model parameters. This results in poor loss estimates, which is intuitive.
\input{figbetaonline}
\input{figcorrectedlossestimatesv1}

\input{figcorrectedsamplev1}
\textbf{Impact of data similarity and privacy factor:} In Fig.\ref{fig:dynamic}, we analyze the influence of strategic participation on data valuation and normalized payments in the regression market. Agent $\{a_{4}\}$ has intermittent connectivity and leaves the regression market after a few iterations. Following \textbf{Case I}, the Shapley valuation is obtained based on all estimation subsets and allocated to the agents. This results in higher payment (up to $49\%$) for even those agents, i.e., $\{a_{4}\}$, particularly without contributing, as compared with \textbf{Case II}. \textbf{Case II} accounts for strategic participation but still allocates payment following the recursive nature of the contribution evaluation procedure. Alg. 1 account for the participation variable in evaluating the performance improvement, followed by the price allocation. Hence, the contribution made by the remaining agents is only considered, offering a gain of two factors. 

\begin{figure}[t!]
    \centering
    \resizebox{0.45\textwidth}{!}{
    \input{figmarketutilityref10}
    }
    \resizebox{0.45\textwidth}{!}{
    \input{figmarketutilityref40}
    }
    \caption{Performance comparison of regression market's utility defined for the central agent while varying reference privacy factor: (top) $\epsilon_{\textrm{ref}} =\ln{10}$, and (bottom) $\epsilon_{\textrm{ref}} =\ln{60}$.}
    \label{fig:marketutility}
\end{figure}
Fig.~\ref{fig:marketutility} shows the performance gain in terms of the overall market utility compared to the baselines for different asked privacy budgets. We evaluate the mechanism for two extreme privacy budgets: $\epsilon_{\textrm{ref}} = \{\ln{10}, \ln{60}\}$, which allows a comparative performance evaluation of the proposed method against baselines participation strategies. First, the central agent relaxing privacy budget enforces supporting agents with better privacy preferences to opt-out from the market; hence, we observe a reduction in the market utility for $\epsilon_{\textrm{ref}} = \ln{60}$ in Fig.~\ref{fig:marketutility}. This follows our analysis for the drop in the loss estimates under such a scenario, as in Fig.~\ref{fig:loss_estimate_v1}. Second, we observed the naive scenario of Case I, which performed the worst of all, where the agents were not strategic in participating but were offered higher pricing. Our approach considers the heterogeneous privacy preferences of the supporting agents and offers pricing to ensure the supply of quality data. Consequently, this improves the overall market utility, as seen in Fig.~\ref{fig:marketutility} for $\epsilon_{\textrm{ref}} = \ln{10}$. This follows improved contributions of the supporting agents, considering the adjustment in privacy factor, as defined in Alg.~\ref{alg:stackelberggame}, line 12. 

Finally, in Fig.~\ref{fig:convergence}, we analyze the number of iterations required for the convergence of the proposed approach under the influence of different degrees of data similarity and against different privacy budgets. Following the original underlying model (as in Fig.~\ref{fig:contribution}), we vary the data similarity measure $\rho$ between the agents $\{a_3, a_4\}$. We observe with the increase in the privacy budget, i.e., lower sensitivity towards data privacy; the algorithm takes more iterations to converge following the tight privacy preference used for the agents. Furthermore, involving the participation of all agents, the results show the adversarial influence of data similarity on the convergence iterations. Statistical information leakage due to data similarity lowers the number of iterations where the correlated feature is less, contributing to the loss estimates. This supports the analysis made in Fig.~\ref{fig:contribution} regarding data contribution in solving the regression problem under data similarity. 
\input{figconvergence}
\section{Conclusion}\label{sec:conclusion}
In this work, we analyzed the interplay between data pricing, privacy, and learning. In particular, we showed the impact of statistical information leakage (for instance, due to data similarity, e.g., data correlation) on the offered pricing and its influence on the value of traded data in a regression data market setting. We proposed a holistic market design where we account for such dependencies. Therein, we developed a query-response strategy for a leader-followers data acquisition mechanism that enjoys a local differential privacy technique, where participation in the trading of data samples happens between a number of data owners and the learner in an elastic fashion. We modelled the strategic interactions between the privacy-aware data owners and the learner as a Stackelberg game and evaluated the consequences of data similarity in terms of participation and the value of traded data on the online regression data market setup.

\bibliographystyle{ieeetr}
\bibliography{references}
	\vskip -2.2\baselineskip plus -1fil
	\begin{IEEEbiography}[{\includegraphics[width=1in,height=1.25in,clip,keepaspectratio]{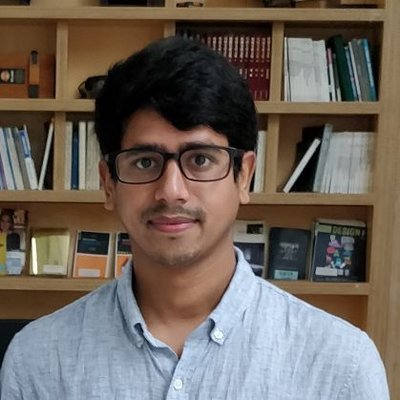}}]
		{\bf Shashi Raj Pandey} (Member, IEEE) received his  B.E. degree in Electrical and Electronics with a specialization in communication engineering from Kathmandu University, Nepal, and the Ph.D. degree in Computer Science and Engineering from Kyung Hee University, Seoul, South Korea. He is currently an Assistant Professor at the Department of Electronics Systems, at Aalborg University. Prior, he was a Postdoctoral Researcher at the Connectivity Section, Aalborg University, from 2021 to 2023. He served as a Network Engineer at Huawei Technologies Nepal Co. Pvt. Ltd, Nepal, from 2013 to 2016.  His research interests include network economics, game theory, wireless communications, data markets and distributed machine learning. He was the recipient of the Best Paper Award at several conferences, including IEICE APNOMS 2019. He was a Member at Large of the IEEE Communication Society Young Professionals 2020 -- 2021. He currently serves as a  Member at Large in the IEEE Communication Society On-Line Content Board and is on the editorial advisory board of IEEE's Spectrum The Institute 2022 -- 2024. He is an affiliated member of the Pioneer Center for AI, Denmark.
	\end{IEEEbiography}
	\vskip -2.2\baselineskip plus -1fil
	\begin{IEEEbiography}[{\includegraphics[width=1.2in,height=1.25in,clip,keepaspectratio]{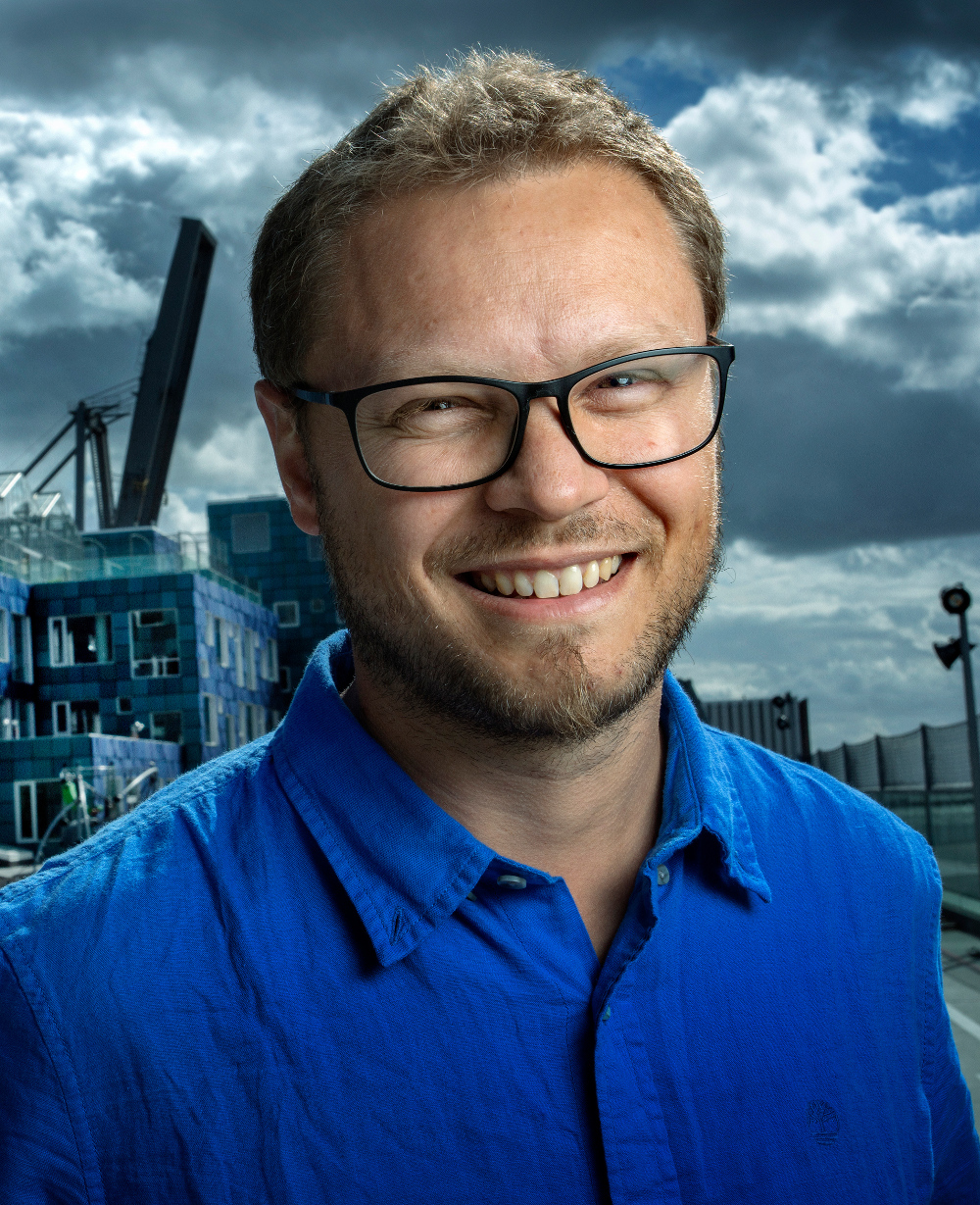}}]
		{\bf Pierre Pinson} (Fellow, IEEE) received the M.Sc. degree in applied mathematics from the National Institute of Applied Sciences (INSA), Toulouse, France, in 2002 and the Ph.D. degree in energetics from Ecole des Mines de Paris, France, in 2006. He is the chair of data-centric design engineering at Imperial College London, United Kingdom, Dyson School of Design Engineering. He is also an affiliated professor of operations research and analytics with the Technical University of Denmark and a chief scientist at Halfspace (Denmark). He is the editor-in-chief of the International Journal of Forecasting. His research interests include analytics, forecasting, optimization and game theory, with application to energy systems mostly, but also logistics, weather-driven industries and business analytics. He is a Fellow of the IEEE, an INFORMS member and a director of the International Institute of Forecasters (IIF).
	\end{IEEEbiography}
	\vskip -2.2\baselineskip plus -1fil
	\begin{IEEEbiography}[{\includegraphics[width=1in,height=1.25in,clip,keepaspectratio]{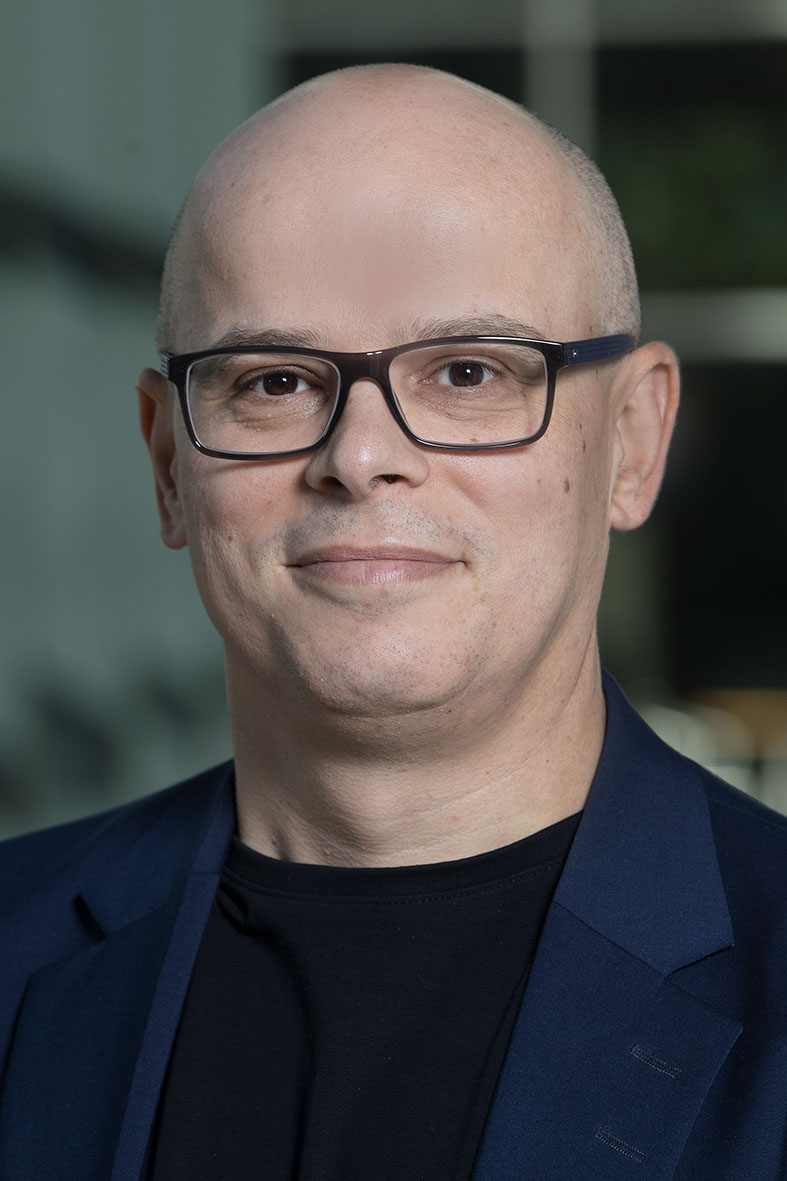}}]
		{\bf Petar Popovski} (Fellow, IEEE) is a Professor at Aalborg University, where he heads the section on Connectivity and a Visiting Excellence Chair at the University of Bremen. He received his Dipl.-Ing and M. Sc. degrees in communication engineering from the University of Sts. Cyril and Methodius in Skopje and the Ph.D. degree from Aalborg University in 2005. He received an ERC Consolidator Grant (2015), the Danish Elite Researcher award (2016), IEEE Fred W. Ellersick prize (2016), IEEE Stephen O. Rice prize (2018), Technical Achievement Award from the IEEE Technical Committee on Smart Grid Communications (2019), the Danish Telecommunication Prize (2020) and Villum Investigator Grant (2021). He was a Member at Large at the Board of Governors in IEEE Communication Society 2019-2021. He is currently an Editor-in-Chief of IEEEE JOURNAL ON SELECTED AREAS IN COMMUNICATIONS. He also serves as a Vice-Chair of the IEEE Communication Theory Technical Committee and the Steering Committee of IEEE TRANSACTIONS ON GREEN COMMUNICATIONS AND NETWORKING. Prof. Popovski was the General Chair for IEEE SmartGridComm 2018 and IEEE Communication Theory Workshop 2019. His research interests are in the area of wireless communication and communication theory. He authored the book ``Wireless Connectivity: An Intuitive and Fundamental Guide'', published by Wiley in 2020.
	\end{IEEEbiography}	
\end{document}

%% file: figutilityCA.tex
\begin{figure}[t!]
    \centering
\begin{tikzpicture}

\definecolor{darkgray176}{RGB}{176,176,176}
\definecolor{darkturquoise0191191}{RGB}{0,191,191}
\definecolor{green01270}{RGB}{0,127,0}
\definecolor{lightgray204}{RGB}{204,204,204}

\begin{axis}[
legend cell align={left},
legend style={
  fill opacity=0.8,
  draw opacity=1,
  text opacity=1,
  at={(0.03,0.97)},
  anchor=north west,
  draw=lightgray204
},
tick align=outside,
tick pos=left,
x grid style={darkgray176},
xlabel={\(\displaystyle \epsilon \)},
xmin=-0.495, xmax=10.395,
xtick style={color=black},
y grid style={darkgray176},
ylabel={\(\displaystyle U(\epsilon)\) },
ymin=-0.0875318378636767, ymax=1.83816859513721,
ytick style={color=black}
]

\addplot [line width=0.48pt, blue]
table {%
0 0
0.1 0.00591176044830896
0.2 0.0116537816247951
0.3 0.0172355392482104
0.4 0.0226657370614007
0.5 0.0279523884750317
0.6 0.0331028876955146
0.7 0.0381240719217299
0.8 0.043022275923389
0.9 0.0478033800941
1 0.0524728528934983
1.1 0.0570357884467325
1.2 0.0614969399495921
1.3 0.0658607494285202
1.4 0.070131374322634
1.5 0.0743127112864966
1.6 0.0784084175552048
1.7 0.0824219301653667
1.8 0.0863564832851076
1.9 0.0902151238720434
2 0.0940007258491472
2.1 0.0977160029637342
2.2 0.10136352047369
2.3 0.104945705786996
2.4 0.108464858165072
2.5 0.111923157587085
2.6 0.115322672860799
2.7 0.118665369055547
2.8 0.121953114324179
2.9 0.125187686173299
3 0.128370777234479
3.1 0.131504000583359
3.2 0.134588894648485
3.3 0.13762692774728
3.4 0.140619502282623
3.5 0.143567958630063
3.6 0.146473578742645
3.7 0.149337589497595
3.8 0.152161165806752
3.9 0.154945433510474
4 0.157691472072854
4.1 0.160400317094406
4.2 0.163072962656839
4.3 0.16571036351323
4.4 0.168313437135644
4.5 0.170883065631214
4.6 0.173420097536677
4.7 0.175925349500513
4.8 0.178399607861022
4.9 0.180843630127977
5 0.183258146374831
5.1 0.185643860547886
5.2 0.188001451698294
5.3 0.190331575142289
5.4 0.192634863554601
5.5 0.194911927999626
5.6 0.197163358904553
5.7 0.199389726978322
5.8 0.201591584079996
5.9 0.203769464039849
6 0.205923883436232
6.1 0.208055342331029
6.2 0.210164324966352
6.3 0.212251300424868
6.4 0.214316723256038
6.5 0.216361034070346
6.6 0.218384660103463
6.7 0.220388015752157
6.8 0.222371503083626
6.9 0.224335512319821
7 0.22628042229822
7.1 0.228206600910412
7.2 0.230114405519764
7.3 0.232004183359351
7.4 0.233876271911263
7.5 0.235730999268329
7.6 0.237568684479211
7.7 0.239389637877794
7.8 0.241194161397722
7.9 0.242982548872854
8 0.244755086324423
8.1 0.24651205223557
8.2 0.248253717813927
8.3 0.249980347242867
8.4 0.251692197922001
8.5 0.253389520697465
8.6 0.255072560082522
8.7 0.256741554468958
8.8 0.25839673632973
8.9 0.260038332413296
9 0.261666563930036
9.1 0.263281646731145
9.2 0.264883791480361
9.3 0.266473203818867
9.4 0.268050084523697
9.5 0.269614629659939
9.6 0.271167030727037
9.7 0.272707474799455
9.8 0.274236144661969
9.9 0.27575321893982
};
\addlegendentry{$\alpha=0.3$, $\beta=-0.2$}
\addplot [line width=0.48pt, green01270]
table {%
0 0
0.1 0.0176067541667098
0.2 0.034471078496421
0.3 0.0506530603733465
0.4 0.0662057753910294
0.5 0.0811763375986761
0.6 0.0956067601882
0.7 0.109534666251891
0.8 0.122993879899184
0.9 0.136014921114284
1 0.148625422572993
1.1 0.16085048273706
1.2 0.172712966570215
1.3 0.184233762931698
1.4 0.195432005927468
1.5 0.206325266110812
1.6 0.216929716330145
1.7 0.227260276154104
1.8 0.237330738111094
1.9 0.247153878423759
2 0.256741554468958
2.1 0.266104790828164
2.2 0.27525385549456
2.3 0.284198327557823
2.4 0.292947157485291
2.5 0.301508720950552
2.6 0.309890867020947
2.7 0.31810096139884
2.8 0.326145925313678
2.9 0.334032270579445
3 0.341766131262427
3.1 0.349353292345152
3.2 0.356799215722044
3.3 0.364109063819437
3.4 0.371287721095772
3.5 0.37833981364628
3.6 0.385269727109202
3.7 0.392081623047123
3.8 0.398779453956644
3.9 0.405366977041951
4 0.411847766872463
4.1 0.418225227031318
4.2 0.424502600849736
4.3 0.430682981312024
4.4 0.436769320206925
4.5 0.442764436593122
4.6 0.448671024639642
4.7 0.454491660895767
4.8 0.460228811039528
4.9 0.465884836149053
5 0.471461998536658
5.1 0.476962467181788
5.2 0.482388322795444
5.3 0.487741562545695
5.4 0.493024104471139
5.5 0.498237791606702
5.6 0.503384395844002
5.7 0.508465621546504
5.8 0.513483108937916
5.9 0.518438437280697
6 0.523333127860071
6.1 0.528168646787673
6.2 0.532946407637734
6.3 0.537667773927688
6.4 0.542334061454073
6.5 0.546946540493749
6.6 0.55150643787964
6.7 0.55601493895949
6.8 0.560473189445457
6.9 0.564882297161757
7 0.569243333697043
7.1 0.573557335967668
7.2 0.577825307697547
7.3 0.582048220819897
7.4 0.586227016805759
7.5 0.590362607923831
7.6 0.594455878435843
7.7 0.598507685731385
7.8 0.602518861405835
7.9 0.606490212284769
8 0.61042252139802
8.1 0.614316548906321
8.2 0.618173032983275
8.3 0.621992690655217
8.4 0.625776218601346
8.5 0.62952429391638
8.6 0.633237574837799
8.7 0.636916701439652
8.8 0.640562296294735
8.9 0.644174965106874
9 0.647755297314907
9.1 0.651303866669875
9.2 0.65482123178683
9.3 0.658307936672596
9.4 0.661764511230733
9.5 0.665191471744852
9.6 0.668589321341421
9.7 0.671958550433065
9.8 0.675299637143356
9.9 0.678613047714014
};
\addlegendentry{$\alpha=0.45$, $\beta=-0.4$}
\addplot [line width=0.48pt, red]
table {%
0 0
0.1 0.0264101312500647
0.2 0.0517066177446314
0.3 0.0759795905600197
0.4 0.099308663086544
0.5 0.121764506398014
0.6 0.1434101402823
0.7 0.164301999377837
0.8 0.184490819848776
0.9 0.204022381671426
1 0.22293813385949
1.1 0.24127572410559
1.2 0.259069449855323
1.3 0.276350644397546
1.4 0.293148008891203
1.5 0.309487899166218
1.6 0.325394574495217
1.7 0.340890414231156
1.8 0.355996107166641
1.9 0.370730817635639
2 0.385112331703437
2.1 0.399157186242246
2.2 0.412880783241841
2.3 0.426297491336735
2.4 0.439420736227936
2.5 0.452263081425828
2.6 0.464836300531421
2.7 0.477151442098259
2.8 0.489218887970517
2.9 0.501048405869167
3 0.512649196893641
3.1 0.524029938517727
3.2 0.535198823583066
3.3 0.546163595729155
3.4 0.556931581643657
3.5 0.56750972046942
3.6 0.577904590663803
3.7 0.588122434570685
3.8 0.598169180934966
3.9 0.608050465562927
4 0.617771650308695
4.1 0.627337840546977
4.2 0.636753901274605
4.3 0.646024471968036
4.4 0.655153980310388
4.5 0.664146654889683
4.6 0.673006536959464
4.7 0.68173749134365
4.8 0.690343216559293
4.9 0.69882725422358
5 0.707192997804988
5.1 0.715443700772682
5.2 0.723582484193165
5.3 0.731612343818543
5.4 0.739536156706709
5.5 0.747356687410053
5.6 0.755076593766003
5.7 0.762698432319755
5.8 0.770224663406874
5.9 0.777657655921045
6 0.784999691790107
6.1 0.79225297018151
6.2 0.799419611456601
6.3 0.806501660891532
6.4 0.813501092181109
6.5 0.820419810740623
6.6 0.827259656819459
6.7 0.834022408439235
6.8 0.840709784168185
6.9 0.847323445742635
7 0.853865000545564
7.1 0.860336003951502
7.2 0.86673796154632
7.3 0.873072331229846
7.4 0.879340525208639
7.5 0.885543911885747
7.6 0.891683817653764
7.7 0.897761528597077
7.8 0.903778292108752
7.9 0.909735318427153
8 0.915633782097029
8.1 0.921474823359481
8.2 0.927259549474913
8.3 0.932989035982825
8.4 0.938664327902019
8.5 0.94428644087457
8.6 0.949856362256699
8.7 0.955375052159478
8.8 0.960843444442102
8.9 0.96626244766031
9 0.971632945972361
9.1 0.976955800004813
9.2 0.982231847680244
9.3 0.987461905008894
9.4 0.992646766846099
9.5 0.997787207617278
9.6 1.00288398201213
9.7 1.0079378256496
9.8 1.01294945571503
9.9 1.01791957157102
};
\addlegendentry{$\alpha=0.45$, $\beta=-0.6$}
\addplot [line width=0.48pt, darkturquoise0191191]
table {%
0 0
0.1 0.0615688329089026
0.2 0.118736004094619
0.3 0.172089103693556
0.4 0.222105389278624
0.5 0.26917778929697
0.6 0.313633670220819
0.7 0.355748657009157
0.8 0.395756993468886
0.9 0.433859432660289
1 0.470229331921695
1.1 0.505017421473486
1.2 0.538355578593941
1.3 0.5703598462849
1.4 0.601132870947137
1.5 0.630765888291416
1.6 0.65934035437308
1.7 0.686929295230015
1.8 0.713598431444088
1.9 0.739407121218666
2 0.764409156021949
2.1 0.788653435618212
2.2 0.812184543783247
2.3 0.835043241738492
2.4 0.857266893024152
2.5 0.878889830934488
2.6 0.899943677588387
2.7 0.920457622079057
2.8 0.940458663843391
2.9 0.959971826342718
3 0.979020345297693
3.1 0.997625835029107
3.2 1.01580843589115
3.3 1.03358694531892
3.4 1.05097893462804
3.5 1.06800085338587
3.6 1.08466812290815
3.7 1.10099522021311
3.8 1.11699575357845
3.9 1.13268253068915
4 1.14806762023146
4.1 1.16316240767496
4.2 1.17797764588875
4.3 1.19252350115531
4.4 1.20680959507531
4.5 1.22084504279604
4.6 1.23463848794364
4.7 1.24819813459466
4.8 1.26153177658283
4.9 1.27464682440337
5 1.28755032994728
5.1 1.30024900927231
5.2 1.31274926359478
5.3 1.32505719866636
5.4 1.33717864268284
5.5 1.34911916285618
5.6 1.36088408076794
5.7 1.37247848660999
5.8 1.38390725240797
5.9 1.39517504431337
6 1.4062863340419
6.1 1.41724540952843
6.2 1.42805638486181
6.3 1.43872320955738
6.4 1.44924967721939
6.5 1.45963943364084
6.6 1.46989598438409
6.7 1.48002270188162
6.8 1.49002283209301
6.9 1.49989950075085
7 1.5096557192259
7.1 1.51929439003884
7.2 1.5288183120439
7.3 1.53823018530757
7.4 1.54753261570366
7.5 1.55672811924425
7.6 1.5658191261645
7.7 1.57480798477804
7.8 1.5836969651181
7.9 1.59248826237858
8 1.6011840001681
8.1 1.60978623358911
8.2 1.61829695215315
8.3 1.62671808254274
8.4 1.63505149122935
8.5 1.64329898695644
8.6 1.65146232309583
8.7 1.65954319988503
8.8 1.6675432665527
8.9 1.67546412333887
9 1.68330732341617
9.1 1.69107437471773
9.2 1.69876674167729
9.3 1.70638584688629
9.4 1.71393307267298
9.5 1.72140976260757
9.6 1.72881722293781
9.7 1.73615672395864
9.8 1.74342950131964
9.9 1.75063675727353
};
\addlegendentry{$\alpha=0.8$, $\beta=-0.8$}
\end{axis}
\end{tikzpicture}
    \caption{An illustration of learner's valuation $U(\epsilon)$ for asked data privacy factor $\epsilon$.}
    \label{fig:CA_utility}
\end{figure}
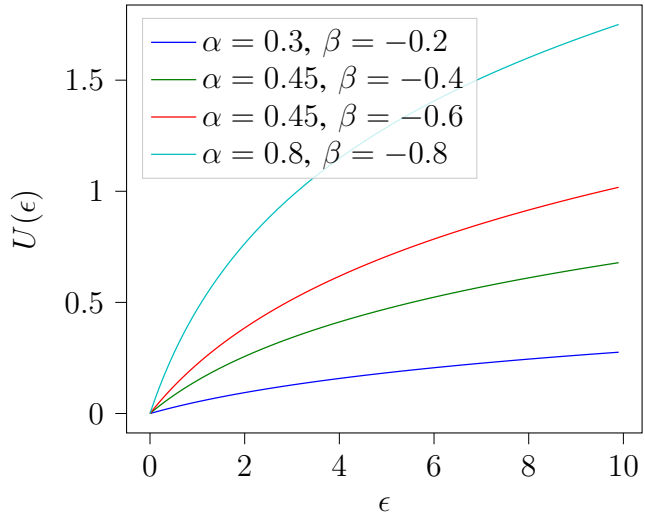

%% file: figleakageimpact.tex
\begin{figure*}[t!]
    \centering
\begin{tikzpicture}[scale=0.9]

\definecolor{darkgray176}{RGB}{176,176,176}

\begin{groupplot}[group style={group size=3 by 1, horizontal sep=90pt}, height=5cm,width=5cm]
\nextgroupplot[
colorbar,
colorbar style={ylabel={}},
colormap={mymap}{[1pt]
  rgb(0pt)=(0.403921568627451,0,0.12156862745098);
  rgb(1pt)=(0.698039215686274,0.0941176470588235,0.168627450980392);
  rgb(2pt)=(0.83921568627451,0.376470588235294,0.301960784313725);
  rgb(3pt)=(0.956862745098039,0.647058823529412,0.509803921568627);
  rgb(4pt)=(0.992156862745098,0.858823529411765,0.780392156862745);
  rgb(5pt)=(0.968627450980392,0.968627450980392,0.968627450980392);
  rgb(6pt)=(0.819607843137255,0.898039215686275,0.941176470588235);
  rgb(7pt)=(0.572549019607843,0.772549019607843,0.870588235294118);
  rgb(8pt)=(0.262745098039216,0.576470588235294,0.764705882352941);
  rgb(9pt)=(0.129411764705882,0.4,0.674509803921569);
  rgb(10pt)=(0.0196078431372549,0.188235294117647,0.380392156862745)
},
point meta max=0.890276010206599,
point meta min=0.0766371734112003,
tick align=outside,
tick pos=left,
x grid style={darkgray176},
xlabel={\(\displaystyle \rho_{3,4}\)},
xmin=0, xmax=11,
xtick style={color=black},
xtick={0.5,2.5,4.5,6.5,8.5,10.5},
xticklabels={0,0.2,0.4,0.6,0.8,1},
y dir=reverse,
y grid style={darkgray176},
ylabel={\(\displaystyle \sigma_4\)},
ymin=0, ymax=11,
ytick style={color=black},
ytick={10.5, 9.5, 8.5, 7.5, 6.5, 5.5, 4.5, 3.5, 2.5, 1.5, 0.5},
yticklabel style={rotate=0.0},
yticklabels={0,0.1,0.2,0.3,0.4,0.5,0.6,0.7,0.8,0.9,1}
]
\addplot graphics [includegraphics cmd=\pgfimage,xmin=0, xmax=11, ymin=11, ymax=0] {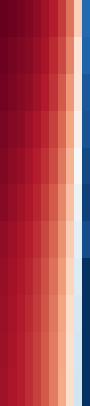};

\nextgroupplot[
colorbar,
colorbar style={ylabel={}},
colormap={mymap}{[1pt]
  rgb(0pt)=(0.403921568627451,0,0.12156862745098);
  rgb(1pt)=(0.698039215686274,0.0941176470588235,0.168627450980392);
  rgb(2pt)=(0.83921568627451,0.376470588235294,0.301960784313725);
  rgb(3pt)=(0.956862745098039,0.647058823529412,0.509803921568627);
  rgb(4pt)=(0.992156862745098,0.858823529411765,0.780392156862745);
  rgb(5pt)=(0.968627450980392,0.968627450980392,0.968627450980392);
  rgb(6pt)=(0.819607843137255,0.898039215686275,0.941176470588235);
  rgb(7pt)=(0.572549019607843,0.772549019607843,0.870588235294118);
  rgb(8pt)=(0.262745098039216,0.576470588235294,0.764705882352941);
  rgb(9pt)=(0.129411764705882,0.4,0.674509803921569);
  rgb(10pt)=(0.0196078431372549,0.188235294117647,0.380392156862745)
},
point meta max=0.353574061346551,
point meta min=0.0273363679977309,
tick align=outside,
tick pos=left,
x grid style={darkgray176},
xlabel={\(\displaystyle \rho_{3,4}\)},
xmin=0, xmax=11,
xtick style={color=black},
xtick={0.5,2.5,4.5,6.5,8.5,10.5},
xticklabels={0,0.2,0.4,0.6,0.8,1},
y dir=reverse,
y grid style={darkgray176},
ylabel={\(\displaystyle \sigma_4\)},
ymin=0, ymax=11,
ytick style={color=black},
ytick={10.5, 9.5, 8.5, 7.5, 6.5, 5.5, 4.5, 3.5, 2.5, 1.5, 0.5},
yticklabel style={rotate=0.0},
yticklabels={0,0.1,0.2,0.3,0.4,0.5,0.6,0.7,0.8,0.9,1}
]

\addplot graphics [includegraphics cmd=\pgfimage,xmin=0, xmax=11, ymin=11, ymax=0] {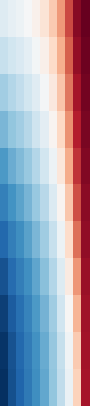};

\nextgroupplot[
colorbar,
colorbar style={ylabel={}},
colormap={mymap}{[1pt]
  rgb(0pt)=(0.403921568627451,0,0.12156862745098);
  rgb(1pt)=(0.698039215686274,0.0941176470588235,0.168627450980392);
  rgb(2pt)=(0.83921568627451,0.376470588235294,0.301960784313725);
  rgb(3pt)=(0.956862745098039,0.647058823529412,0.509803921568627);
  rgb(4pt)=(0.992156862745098,0.858823529411765,0.780392156862745);
  rgb(5pt)=(0.968627450980392,0.968627450980392,0.968627450980392);
  rgb(6pt)=(0.819607843137255,0.898039215686275,0.941176470588235);
  rgb(7pt)=(0.572549019607843,0.772549019607843,0.870588235294118);
  rgb(8pt)=(0.262745098039216,0.576470588235294,0.764705882352941);
  rgb(9pt)=(0.129411764705882,0.4,0.674509803921569);
  rgb(10pt)=(0.0196078431372549,0.188235294117647,0.380392156862745)
},
point meta max=0.516315158325074,
point meta min=0.0280186920989197,
tick align=outside,
tick pos=left,
x grid style={darkgray176},
xlabel={\(\displaystyle \rho_{3,4}\)},
xmin=0, xmax=11,
xtick style={color=black},
xtick={0.5,2.5,4.5,6.5,8.5,10.5},
xticklabels={0,0.2,0.4,0.6,0.8,1},
y dir=reverse,
y grid style={darkgray176},
ylabel={\(\displaystyle \sigma_4\)},
ymin=0, ymax=11,
ytick style={color=black},
ytick={10.5, 9.5, 8.5, 7.5, 6.5, 5.5, 4.5, 3.5, 2.5, 1.5, 0.5},
yticklabel style={rotate=0},
yticklabels={0,0.1,0.2,0.3,0.4,0.5,0.6,0.7,0.8,0.9,1}
]
\addplot graphics [includegraphics cmd=\pgfimage,xmin=0, xmax=11, ymin=11, ymax=0] {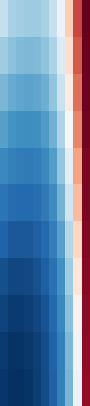};
\end{groupplot}
\node (agent 1)      at ( 1.5,-1.5) {Agent $\{a_2\}$}; 
\node (agent 2)      at ( 7.5,-1.5) {Agent $\{a_3\}$};
\node (agent 3)      at ( 13.5,-1.5) {Agent $\{a_4\}$};
\end{tikzpicture}
    \caption{Example scenario: heatmap represents the impact on the normalized contribution of each agent given information leakage due to data correlation $\rho_{3,4}$ between $\{a_3\} - \{a_4\}$ and the noise injection $\sigma_4$ by $\{a_4\}$.}
    \label{fig:contribution}
\end{figure*}

%% file: algo.tex
  	\begin{algorithm}[t!]
        	\caption{\strut First-order Iterative Backward Induction}
        	\label{alg:stackelberggame}
        	\begin{algorithmic}[1]
        		\STATE{Start with random sample $\epsilon_{\textrm{ref}} < \epsilon \sim f_\varepsilon$, offered pricing signal $p= p_{\max}$, normalizing parameters $\gamma, \{ \psi_n, \varphi_n\},\forall n$ set to 1, $\zeta_{\textrm{ref}}=0.9$}.
        		\STATE{$\mathcal{P} = \{\}, \mathcal{A}\leftarrow{\{n | q_n = 1, \forall n\}}$};
        		 \REPEAT
        		  \STATE{$\mathcal{R} = \mathcal{A}$};
        		   \STATE {Evaluate the performance improvement $L(\zeta)$};  
        		   \STATE Solve the following optimization problem:\\
        		   $$\underset{p> \tilde{p}}{\textrm{maximize}}\; L(\zeta)\frac{1}{{\ln[\alpha \epsilon^*p(\epsilon^*) + 1]}^\beta} - p(\epsilon^*) \sum\nolimits_{n \in \mathcal{A}} q_n(\epsilon^*);$$
        		    \FORALL{agents $n\in \mathcal{A}$}
                    \STATE {Invoke proportionally fair price allocation $p_{a_n}$ with the marginal contribution \eqref{eq:shapTMC}}; 
                    \STATE {$\mathcal{P}\leftarrow{\mathcal{P}}+ \{p_{a_n}\}$}
                    \STATE {$\mathcal{R} = \mathcal{A}\setminus n$};
                    \ENDFOR
            	     \STATE {Set $\epsilon = \max\{\epsilon_n(q_n^*), \epsilon_{\textrm{ref}}\}, \tilde{p}=\sup\mathcal{P}$};         
        	     \UNTIL{$\mathcal{A}= \{\emptyset \}$};
        	\end{algorithmic}
        	\label{Algorithm}
        \end{algorithm}

%% file: figcumpay1.tex
 \begin{figure*}
\centering
 \begin{tikzpicture}
    \pgfplotsset{footnotesize,samples=10}
    \begin{groupplot}[group style = {group size = 3 by 1, horizontal sep = 50pt}, width = 5cm, height = 5.0cm]
    \definecolor{darkgray176}{RGB}{176,176,176}
    \definecolor{darkorange25512714}{RGB}{255,127,14}
    \definecolor{forestgreen4416044}{RGB}{44,160,44}
    \definecolor{lightgray204}{RGB}{204,204,204}
    \definecolor{red}{rgb}{0.27, 0.51, 0.71} 
        \nextgroupplot[
            legend style = { column sep = 10pt, legend columns = -1, legend to name=named},
            tick align=outside,
            tick pos=left,
            x grid style={darkgray176},
            xlabel={Time $(t)$},
            xmin=1850, xmax=5150,
            xtick={2000,3000,4000,5000,6000},
            xticklabels={2{,}000,3{,}000,4{,}000,5{,}000,6{,}000},
            xtick style={color=black},
            y grid style={darkgray176},
            ylabel={Normalized payment},
            ymin=0.13779951501631, ymax=1.01366672729828,
            ytick style={color=black}]
            \addplot [semithick, red]
            table {%
            2000 0.179267569820894
            2100 0.177611661029127
            2200 0.177611661029127
            2300 0.18820826156631
            2400 0.196518718054543
            2500 0.211583155461673
            2600 0.211174557022588
            2700 0.262155890317892
            2800 0.274670427706084
            2900 0.271650838194013
            3000 0.246704174733944
            3100 0.249855066133367
            3200 0.211195827315433
            3300 0.205462996674999
            3400 0.196441895267459
            3500 0.211857485956692
            3600 0.236909596253009
            3700 0.23088793994948
            3800 0.213504648045846
            3900 0.218780448182584
            4000 0.199498772398061
            4100 0.195044024110688
            4200 0.220378350479935
            4300 0.240181882205898
            4400 0.228006790530984
            4500 0.229713958260609
            4600 0.194252591126715
            4700 0.211717089569381
            4800 0.218848646697424
            4900 0.217959114485674
            4999 0.217904455723758
            };   
                       
        \addplot [semithick, red, dotted]
        table {%
        2000 0.223882312262389
        2100 0.21546744642783
        2200 0.181305519308572
        2300 0.216395458128339
        2400 0.215760401169291
        2500 0.2304209895183
        2600 0.227705049258736
        2700 0.292594570185988
        2800 0.288713299703007
        2900 0.281375874201522
        3000 0.261741641772859
        3100 0.298782672307367
        3200 0.246873377208341
        3300 0.2515055205688
        3400 0.249329708081159
        3500 0.263200237066979
        3600 0.292817213211854
        3700 0.267428538262323
        3800 0.265862207352312
        3900 0.292409822631041
        4000 0.253587204915063
        4100 0.223208060212724
        4200 0.254128903662876
        4300 0.295916738114341
        4400 0.268099902136381
        4500 0.276442841063682
        4600 0.234598708777452
        4700 0.273253514654323
        4800 0.269653304859046
        4900 0.242722675905223
        4999 0.2335996778749
        };   
        \addplot [semithick, black, const plot mark left, opacity=0.5, dashed]
        table {%
        2000 0.177611661029127
        2500 0.177611661029127
        3000 0.177611661029127
        3500 0.177611661029127
        4000 0.177611661029127
        4500 0.177611661029127
        5000 0.177611661029127
        };
        \legend{$\textrm{Case I}$, $\textrm{Case II}$, $\textrm{Asked Price}$}
        \addplot [semithick, darkorange25512714]
        table {%
        2000 0.528070469607921
        2100 0.50677913342381
        2200 0.416370272587579
        2300 0.530421084562604
        2400 0.531037074888411
        2500 0.557688873918197
        2600 0.546371235806943
        2700 0.633835415220166
        2800 0.62647161444663
        2900 0.609823192129706
        3000 0.573529127779843
        3100 0.62398030467148
        3200 0.543425468151866
        3300 0.535420511942745
        3400 0.504533404599648
        3500 0.495911029884381
        3600 0.534753214826758
        3700 0.52224569799728
        3800 0.529446531900995
        3900 0.591978407588778
        4000 0.533411353430624
        4100 0.499565850507343
        4200 0.553532001981135
        4300 0.633810553334391
        4400 0.584524524002101
        4500 0.591582221971881
        4600 0.530799383213825
        4700 0.588475072269746
        4800 0.583091359116205
        4900 0.539641774636141
        4999 0.526243492681925
        };
        \addplot [semithick, darkorange25512714, dotted]
        table {%
        2000 0.447865490703061
        2100 0.439450624868502
        2200 0.405288697749244
        2300 0.440378636569012
        2400 0.439743579609963
        2500 0.454404167958972
        2600 0.451688227699408
        2700 0.51657774862666
        2800 0.512696478143679
        2900 0.505359052642194
        3000 0.485724820213532
        3100 0.522765850748039
        3200 0.470856555649013
        3300 0.475488699009473
        3400 0.473312886521831
        3500 0.487183415507652
        3600 0.516800391652526
        3700 0.491411716702995
        3800 0.489845385792985
        3900 0.516393001071713
        4000 0.477570383355735
        4100 0.447191238653396
        4200 0.478112082103549
        4300 0.519899916555014
        4400 0.492083080577053
        4500 0.500426019504355
        4600 0.458581887218124
        4700 0.497236693094995
        4800 0.493636483299718
        4900 0.466705854345895
        4999 0.457582856315572
        };
        \addplot [semithick, black, const plot mark left, opacity=0.5, dashed]
        table {%
        2000 0.401594839469799
        2500 0.401594839469799
        3000 0.401594839469799
        3500 0.401594839469799
        4000 0.401594839469799
        4500 0.401594839469799
        5000 0.401594839469799
        };
        \addplot [semithick, forestgreen4416044]
        table {%
        2000 0.762439749138559
        2100 0.7517275307762
        2200 0.705488683135398
        2300 0.72120102590226
        2400 0.709734351252026
        2500 0.726660468211145
        2600 0.727522943723232
        2700 0.848635514723712
        2800 0.82795969617713
        2900 0.818278006000187
        3000 0.800981804095468
        3100 0.895543857942438
        3200 0.807120752883881
        3300 0.839387113175275
        3400 0.870592071975344
        3500 0.919280971944662
        3600 0.973854581285466
        3700 0.89082905462035
        3800 0.894746188980226
        3900 0.933128974270618
        4000 0.855687233349386
        4100 0.772470905750683
        4200 0.816853801708253
        4300 0.883923056434894
        4400 0.834116833530256
        4500 0.858723723540057
        4600 0.787591400287087
        4700 0.867070436295981
        4800 0.850921753140373
        4900 0.787538354016893
        4999 0.764499302611735
        };
        \addplot [semithick, forestgreen4416044, dotted]
        table {%
        2000 0.751759334368661
        2100 0.743344468534102
        2200 0.709182541414843
        2300 0.744272480234611
        2400 0.743637423275562
        2500 0.758298011624571
        2600 0.755582071365008
        2700 0.82047159229226
        2800 0.816590321809278
        2900 0.809252896307794
        3000 0.789618663879131
        3100 0.826659694413639
        3200 0.774750399314612
        3300 0.779382542675072
        3400 0.77720673018743
        3500 0.791077259173251
        3600 0.820694235318125
        3700 0.795305560368595
        3800 0.793739229458584
        3900 0.820286844737312
        4000 0.781464227021335
        4100 0.751085082318996
        4200 0.782005925769148
        4300 0.823793760220613
        4400 0.795976924242653
        4500 0.804319863169954
        4600 0.762475730883724
        4700 0.801130536760594
        4800 0.797530326965318
        4900 0.770599698011494
        4999 0.761476699981172
        };
        \addplot [semithick, black, const plot mark left, opacity=0.5, dashed]
        table {%
        2000 0.705488683135398
        2500 0.705488683135398
        3000 0.705488683135398
        3500 0.705488683135398
        4000 0.705488683135398
        4500 0.705488683135398
        5000 0.705488683135398
        };
        \draw (axis cs:2000,0.187611661029127) node[
          scale=0.8,
          anchor=base west,
          text=black,
          rotate=0.0
        ]{Agent $\{a_2\}$};
        \draw (axis cs:2000,0.411594839469799) node[
          scale=0.8,
          anchor=base west,
          text=black,
          rotate=0.0
        ]{Agent $\{a_3\}$};
        \draw (axis cs:2000,0.715488683135398) node[
          scale=0.8,
          anchor=base west,
          text=black,
          rotate=0.0
        ]{Agent $\{a_4\}$};
            \nextgroupplot[
            tick align=outside,
            tick pos=left,
            x grid style={darkgray176},
            xmin=2375, xmax=5125,
            xtick={2000,3000,4000,5000,6000},
            xticklabels={2{,}000,3{,}000,4{,}000,5{,}000,6{,}000},
            xlabel={Time $(t)$},
            xtick style={color=black},
            y grid style={darkgray176},
            ymin=0.692070388227895, ymax=0.987272876192969,
            ytick style={color=black}
            ]
            \addplot [semithick, forestgreen4416044]
            table {%
            2500 0.726660468211145
            2600 0.727522943723232
            2700 0.848635514723712
            2800 0.82795969617713
            2900 0.818278006000187
            3000 0.800981804095468
            3100 0.895543857942438
            3200 0.807120752883881
            3300 0.839387113175275
            3400 0.870592071975344
            3500 0.919280971944662
            3600 0.973854581285466
            3700 0.89082905462035
            3800 0.894746188980226
            3900 0.933128974270618
            4000 0.855687233349386
            4100 0.772470905750683
            4200 0.816853801708253
            4300 0.883923056434894
            4400 0.834116833530256
            4500 0.858723723540057
            4600 0.787591400287087
            4700 0.867070436295981
            4800 0.850921753140373
            4900 0.787538354016893
            4999 0.764499302611735
            };
            \addplot [semithick, forestgreen4416044, dotted]
            table {%
            2500 0.758298011624571
            2600 0.755582071365008
            2700 0.82047159229226
            2800 0.816590321809278
            2900 0.809252896307794
            3000 0.789618663879131
            3100 0.826659694413639
            3200 0.774750399314612
            3300 0.779382542675072
            3400 0.77720673018743
            3500 0.791077259173251
            3600 0.820694235318125
            3700 0.795305560368595
            3800 0.793739229458584
            3900 0.820286844737312
            4000 0.781464227021335
            4100 0.751085082318996
            4200 0.782005925769148
            4300 0.823793760220613
            4400 0.795976924242653
            4500 0.804319863169954
            4600 0.762475730883724
            4700 0.801130536760594
            4800 0.797530326965318
            4900 0.770599698011494
            4999 0.761476699981172
            };
            \addplot [semithick, black, const plot mark left, opacity=0.5, dashed]
            table {%
            2500 0.705488683135398
            3000 0.705488683135398
            3500 0.705488683135398
            4000 0.705488683135398
            4500 0.705488683135398
            5000 0.705488683135398
            };
            \draw[<->,draw=black] (axis cs:3600,0.820694235318125) -- (axis cs:3600,0.958854581285466);
            \draw (axis cs:3610,0.85) node[
              scale=0.8,
              anchor=base west,
              text=black,
              rotate=0.0
            ]{15.3\%};
\nextgroupplot[
tick align=outside,
tick pos=left,
x grid style={darkgray176},
xmin=2375, xmax=5125,
xtick={2000,3000,4000,5000,6000},
xticklabels={2{,}000,3{,}000,4{,}000,5{,}000,6{,}000},
xtick style={color=black},
xlabel={Time $(t)$},
y grid style={darkgray176},
ymin=0.171553110465215, ymax=0.304841222871279,
ytick style={color=black}
]
\addplot [semithick, red]
table {%
2500 0.211583155461673
2600 0.211174557022588
2700 0.262155890317892
2800 0.274670427706084
2900 0.271650838194013
3000 0.246704174733944
3100 0.249855066133367
3200 0.211195827315433
3300 0.205462996674999
3400 0.196441895267459
3500 0.211857485956692
3600 0.236909596253009
3700 0.23088793994948
3800 0.213504648045846
3900 0.218780448182584
4000 0.199498772398061
4100 0.195044024110688
4200 0.220378350479935
4300 0.240181882205898
4400 0.228006790530984
4500 0.229713958260609
4600 0.194252591126715
4700 0.211717089569381
4800 0.218848646697424
4900 0.217959114485674
4999 0.217904455723758
};
\addplot [semithick, red, dotted]
table {%
2500 0.2304209895183
2600 0.227705049258736
2700 0.292594570185988
2800 0.288713299703007
2900 0.281375874201522
3000 0.261741641772859
3100 0.298782672307367
3200 0.246873377208341
3300 0.2515055205688
3400 0.249329708081159
3500 0.263200237066979
3600 0.292817213211854
3700 0.267428538262323
3800 0.265862207352312
3900 0.292409822631041
4000 0.253587204915063
4100 0.223208060212724
4200 0.254128903662876
4300 0.295916738114341
4400 0.268099902136381
4500 0.276442841063682
4600 0.234598708777452
4700 0.273253514654323
4800 0.269653304859046
4900 0.242722675905223
4999 0.2335996778749
};
\addplot [semithick, black, const plot mark left, opacity=0.5, dashed]
table {%
2500 0.177611661029127
3000 0.177611661029127
3500 0.177611661029127
4000 0.177611661029127
4500 0.177611661029127
5000 0.177611661029127
};
\draw[<->,draw=black] (axis cs:3900,0.282409822631041) -- (axis cs:3900,0.218780448182584);
\draw (axis cs:3500,0.25) node[
  scale=0.8,
  anchor=base west,
  text=black,
  rotate=0.0
]{5.6\%};
    \end{groupplot}
\node at ($(group c2r1) + (0,2.2cm)$) {\ref{named}};
\end{tikzpicture}
\caption{(Left) Evolution of normalized payment to agents $\{a_2, a_3, a_4\}$ during training in the online regression market by the agent $\{a_{1}\}$. (Mid -- Right) Zoomed-in illustration depicting price variability for agents $\{a_2, a_4\}$.}
    \label{fig:cumpay}
 \end{figure*}
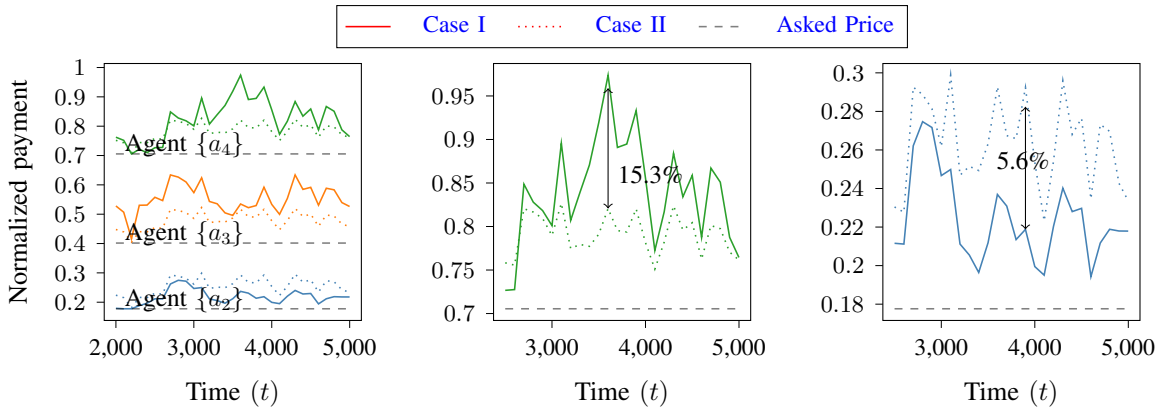

%% file: figcorrectedlossestimatesv1.tex
\begin{figure}[t!]
    \centering
\begin{tikzpicture}

\definecolor{darkgray176}{RGB}{176,176,176}
\definecolor{darkorange25512714}{RGB}{255,127,14}
\definecolor{forestgreen4416044}{RGB}{44,160,44}
\definecolor{lightgray204}{RGB}{204,204,204}
\definecolor{steelblue31119180}{RGB}{31,119,180}

\pgfplotsset{scaled x ticks=false} 
\begin{axis}[
legend cell align={left},
legend style={
  fill opacity=0.8,
  draw opacity=1,
  text opacity=1,
  at={(0.03,0.97)},
  anchor=north west,
  draw=lightgray204
},
tick align=outside,
tick pos=left,
x grid style={darkgray176},
xlabel={Time ($t$)},
xmin=-498.9, xmax=9000,
xtick style={color=black},
y grid style={darkgray176},
ylabel={Loss estimation},
ymin=-0.05, ymax=1.05,
ytick style={color=black}
]
\addplot [semithick, steelblue31119180, mark=asterisk, mark size=3, mark options={solid}]
table {%
1000 0.765817191462275
2000 0.808250887668889
3000 0.891528337964863
4000 0.873591392992582
5000 0.826920304035215
5999 1
6999 0.858162736473324
7999 0.852139055639866
8999 0.905823073085176
9999 0.803510448701511
};
\addlegendentry{Central Info}
\addplot [semithick, darkorange25512714, mark=+, mark size=3, mark options={solid}]
table {%
1000 0.404378098811254
2000 0.440035239128625
3000 0.450971082559675
4000 0.464727157594474
5000 0.41241422931754
5999 0.530528046313661
6999 0.412453156362654
7999 0.476915747088229
8999 0.496879588236749
9999 0.462929353341717
};
\addlegendentry{Partial Info}
\addplot [semithick, forestgreen4416044, mark=triangle*, mark size=3, mark options={solid}]
table {%
1000 1.70013897851681e-28
2000 4.11206349520051e-29
3000 2.89814692738643e-29
4000 3.65545581464523e-29
5000 2.00107021226237e-29
5999 1.03757871998611e-28
6999 1.94334258008542e-29
7999 7.5370525894744e-29
8999 1.66596628237435e-29
9999 6.93297867736758e-29
};
\addlegendentry{Full Info}
\end{axis}
\end{tikzpicture}
\caption{Impact of participation on the normalized loss estimates for different collaborative online learning scenarios with four agents: (i) Central Info, with only agent $\{a_{1}\}$, (ii) Partial Info, with agents $\{a_{1}, a_2, a_3\}$, and (iii) Full Info, with all agents.}
\label{fig:loss_estimate_v1}
\end{figure}
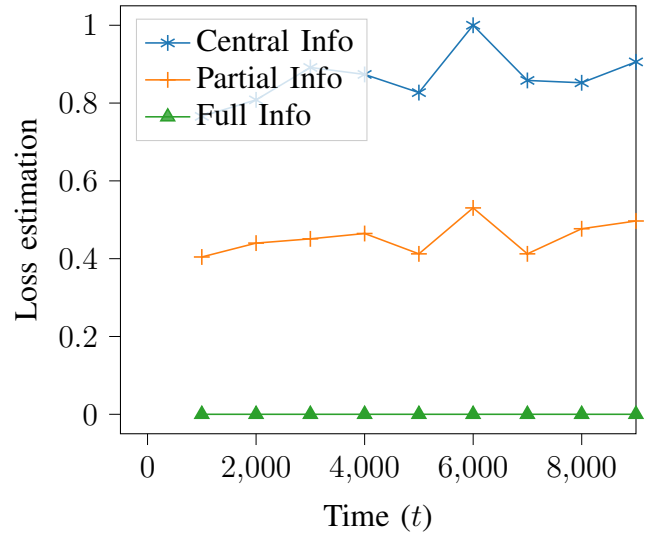

%% file: figcorrectedsamplev1.tex
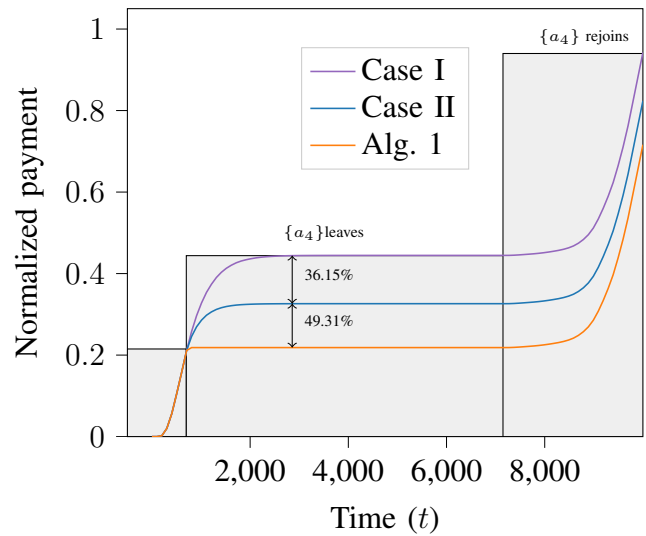
\begin{figure}[t!]
    \centering
\begin{tikzpicture}

\definecolor{darkgray176}{RGB}{176,176,176}
\definecolor{grey}{RGB}{176,176,176}
\definecolor{darkorange25512714}{RGB}{255,127,14}
\definecolor{green01270}{RGB}{0,127,0}
\definecolor{lightgray204}{RGB}{204,204,204}
\definecolor{mediumpurple148103189}{RGB}{148,103,189}
\definecolor{steelblue31119180}{RGB}{31,119,180}

\begin{axis}[
legend cell align={left},
legend style={
  fill opacity=0.8,
  draw opacity=1,
  text opacity=1,
  at={(0.5,0.91)},
  anchor=north,
  draw=lightgray204
},
tick align=outside,
tick pos=left,
x grid style={darkgray176},
xlabel={Time ($t$)},
xmin=-0.05, xmax=1,
xtick={0.2,0.4, 0.6, 0.8},
xticklabels={2{,}000,4{,}000, 6{,}000, 8{,}000},
xtick style={color=black},
y grid style={darkgray176},
ylabel={Normalized payment},
ymin=-0, ymax=1.04999999882597,
ytick style={color=black}
]

\draw[fill=grey!20, opacity=100] (0,0) rectangle (12,21.5); 
\draw[fill=grey!20, opacity=100] (12,0) rectangle (105,44.4); 
\draw[fill=grey!20, opacity=100] (76.5, 0) rectangle (105, 94); 
\node[font=\tiny]  at (40,50) {$\{a_4\} $leaves};
\node[font=\tiny]  at (93, 98) {$\{a_4\}$ rejoins};

\addplot [semithick, mediumpurple148103189]
table {%
0 2.3480567210337e-08
0.0101010101010101 1.66712027193393e-06
0.0202020202020202 0.00224324861256279
0.0303030303030303 0.0199932861016126
0.0404040404040404 0.0562472072808875
0.0505050505050505 0.10567655487735
0.0606060606060606 0.158027106960324
0.0707070707070707 0.21146565157966
0.0808080808080808 0.258021528152889
0.0909090909090909 0.296637261607444
0.101010101010101 0.327973482442195
0.111111111111111 0.353834440478447
0.121212121212121 0.373843376343306
0.131313131313131 0.390474621541621
0.141414141414141 0.403126078536503
0.151515151515152 0.413007620094169
0.161616161616162 0.420466246311457
0.171717171717172 0.426473922787422
0.181818181818182 0.43083759022002
0.191919191919192 0.433982393475131
0.202020202020202 0.436540907288561
0.212121212121212 0.4386271881128
0.222222222222222 0.440167913222208
0.232323232323232 0.441271589756862
0.242424242424242 0.442095777244458
0.252525252525253 0.442750549957391
0.262626262626263 0.443290412777763
0.272727272727273 0.443694689423409
0.282828282828283 0.443963552928834
0.292929292929293 0.444142674392004
0.303030303030303 0.444274175406271
0.313131313131313 0.444357950138618
0.323232323232323 0.44441464947242
0.333333333333333 0.4444589514333
0.343434343434343 0.444494324852875
0.353535353535354 0.444519906312532
0.363636363636364 0.444538803326765
0.373737373737374 0.444553067543398
0.383838383838384 0.444563372250494
0.393939393939394 0.444572246207351
0.404040404040404 0.444578945867274
0.414141414141414 0.444583820587258
0.424242424242424 0.444587477237494
0.434343434343434 0.44459004659739
0.444444444444444 0.444591915827705
0.454545454545455 0.4445933760377
0.464646464646465 0.444594314624415
0.474747474747475 0.44459499050186
0.484848484848485 0.444595433924541
0.494949494949495 0.444595782388554
0.505050505050505 0.44459604131153
0.515151515151515 0.444596235856621
0.525252525252525 0.444596377318882
0.535353535353535 0.444596471213928
0.545454545454546 0.444596537598867
0.555555555555556 0.444596587128469
0.565656565656566 0.444596624520403
0.575757575757576 0.444596649942494
0.585858585858586 0.444596666638587
0.595959595959596 0.444596677427618
0.606060606060606 0.444596685611357
0.616161616161616 0.444596691739886
0.626262626262626 0.444596696831591
0.636363636363636 0.444596700378974
0.646464646464647 0.444596702975689
0.656565656565657 0.444596704710699
0.666666666666667 0.444596705854166
0.676767676767677 0.44459670671119
0.686868686868687 0.444596707343231
0.696969696969697 0.444596707797901
0.707070707070707 0.444596708125954
0.717171717171717 0.444596708363618
0.727272727272727 0.444656746346192
0.737373737373737 0.445100153817202
0.747474747474748 0.445832781104331
0.757575757575758 0.446696980419011
0.767676767676768 0.4476032179199
0.777777777777778 0.448570128248609
0.787878787878788 0.449743957734756
0.797979797979798 0.451189737375838
0.808080808080808 0.452871278734901
0.818181818181818 0.454809310258703
0.828282828282828 0.456947011606251
0.838383838383838 0.459605876047511
0.848484848484849 0.463075744635061
0.858585858585859 0.467806487405791
0.868686868686869 0.474574417993715
0.878787878787879 0.483583086533131
0.888888888888889 0.495424786140343
0.898989898989899 0.511678985063819
0.909090909090909 0.533812742872449
0.919191919191919 0.560244312355208
0.929292929292929 0.58985038537449
0.939393939393939 0.623314480509745
0.94949494949495 0.664694688726268
0.95959595959596 0.710581662282554
0.96969696969697 0.762558362257352
0.97979797979798 0.822341840632632
0.98989898989899 0.882392859292042
1 0.942396336959672
};
\addlegendentry{Case I}
\addplot [semithick, steelblue31119180]
table {%
0 2.3480567210337e-08
0.0101010101010101 1.66712027193393e-06
0.0202020202020202 0.00224324861256279
0.0303030303030303 0.0199932861016126
0.0404040404040404 0.0562472072808875
0.0505050505050505 0.10567655487735
0.0606060606060606 0.158027106960324
0.0707070707070707 0.21146565157966
0.0808080808080808 0.245840792915346
0.0909090909090909 0.269402643130476
0.101010101010101 0.285820656132118
0.111111111111111 0.297809904773871
0.121212121212121 0.305984177651366
0.131313131313131 0.312076374130305
0.141414141414141 0.316231127808991
0.151515151515152 0.319172984640038
0.161616161616162 0.321199050988139
0.171717171717172 0.322694766959481
0.181818181818182 0.323695029107816
0.191919191919192 0.324360092884697
0.202020202020202 0.324861395882932
0.212121212121212 0.325243692531917
0.222222222222222 0.325509892931618
0.232323232323232 0.325689071325761
0.242424242424242 0.32581566023447
0.252525252525253 0.325910886276217
0.262626262626263 0.325986017940829
0.272727272727273 0.326040018281647
0.282828282828283 0.326074492195164
0.292929292929293 0.326096439016148
0.303030303030303 0.326111854680457
0.313131313131313 0.326121245164383
0.323232323232323 0.326127292354379
0.333333333333333 0.326131777481571
0.343434343434343 0.32613520281368
0.353535353535354 0.326137578523927
0.363636363636364 0.326139266414849
0.373737373737374 0.326140494741321
0.383838383838384 0.326141352016904
0.393939393939394 0.326142066089513
0.404040404040404 0.326142591157846
0.414141414141414 0.326142962941593
0.424242424242424 0.326143234665697
0.434343434343434 0.326143420844456
0.444444444444444 0.326143552940315
0.454545454545455 0.326143653812135
0.464646464646465 0.326143717232048
0.474747474747475 0.326143761829599
0.484848484848485 0.326143790412907
0.494949494949495 0.326143812333546
0.505050505050505 0.326143828257687
0.515151515151515 0.326143839973904
0.525252525252525 0.326143848323872
0.535353535353535 0.326143853755258
0.545454545454546 0.326143857514739
0.555555555555556 0.32614386026239
0.565656565656566 0.326143862298218
0.575757575757576 0.326143863656917
0.585858585858586 0.326143864531182
0.595959595959596 0.326143865083835
0.606060606060606 0.326143865493524
0.616161616161616 0.326143865794075
0.626262626262626 0.32614386603933
0.636363636363636 0.326143866207269
0.646464646464647 0.32614386632807
0.656565656565657 0.326143866407438
0.666666666666667 0.326143866458766
0.676767676767677 0.326143866496505
0.686868686868687 0.326143866523839
0.696969696969697 0.326143866543143
0.707070707070707 0.326143866556837
0.717171717171717 0.326143866566592
0.727272727272727 0.326203904407793
0.737373737373737 0.326647311777676
0.747474747474748 0.32737993898905
0.757575757575758 0.328244138249891
0.767676767676768 0.329150375714527
0.777777777777778 0.330117286018431
0.787878787878788 0.331291115486166
0.797979797979798 0.332736895112785
0.808080808080808 0.334418436461281
0.818181818181818 0.336356467977196
0.828282828282828 0.338494169319403
0.838383838383838 0.341153033756892
0.848484848484849 0.344622902341718
0.858585858585859 0.349353645110437
0.868686868686869 0.356121575696835
0.878787878787879 0.365130244235187
0.888888888888889 0.376971943841657
0.898989898989899 0.393226142764589
0.909090909090909 0.4153599005728
0.919191919191919 0.441791470055267
0.929292929292929 0.471397543074361
0.939393939393939 0.504861638209488
0.94949494949495 0.54624184642591
0.95959595959596 0.592128819982122
0.96969696969697 0.644105519956864
0.97979797979798 0.703888998332101
0.98989898989899 0.763940016991482
1 0.823943494659092
};
\addlegendentry{Case II}
\addplot [semithick, darkorange25512714]
table {%
0 2.3480567210337e-08
0.0101010101010101 1.66712027193393e-06
0.0202020202020202 0.00224324861256279
0.0303030303030303 0.0199932861016126
0.0404040404040404 0.0562472072808875
0.0505050505050505 0.10567655487735
0.0606060606060606 0.158027106960324
0.0707070707070707 0.21146565157966
0.0808080808080808 0.218423582161079
0.0909090909090909 0.218423582161079
0.101010101010101 0.218423582161079
0.111111111111111 0.218423582161079
0.121212121212121 0.218423582161079
0.131313131313131 0.218423582161079
0.141414141414141 0.218423582161079
0.151515151515152 0.218423582161079
0.161616161616162 0.218423582161079
0.171717171717172 0.218423582161079
0.181818181818182 0.218423582161079
0.191919191919192 0.218423582161079
0.202020202020202 0.218423582161079
0.212121212121212 0.218423582161079
0.222222222222222 0.218423582161079
0.232323232323232 0.218423582161079
0.242424242424242 0.218423582161079
0.252525252525253 0.218423582161079
0.262626262626263 0.218423582161079
0.272727272727273 0.218423582161079
0.282828282828283 0.218423582161079
0.292929292929293 0.218423582161079
0.303030303030303 0.218423582161079
0.313131313131313 0.218423582161079
0.323232323232323 0.218423582161079
0.333333333333333 0.218423582161079
0.343434343434343 0.218423582161079
0.353535353535354 0.218423582161079
0.363636363636364 0.218423582161079
0.373737373737374 0.218423582161079
0.383838383838384 0.218423582161079
0.393939393939394 0.218423582161079
0.404040404040404 0.218423582161079
0.414141414141414 0.218423582161079
0.424242424242424 0.218423582161079
0.434343434343434 0.218423582161079
0.444444444444444 0.218423582161079
0.454545454545455 0.218423582161079
0.464646464646465 0.218423582161079
0.474747474747475 0.218423582161079
0.484848484848485 0.218423582161079
0.494949494949495 0.218423582161079
0.505050505050505 0.218423582161079
0.515151515151515 0.218423582161079
0.525252525252525 0.218423582161079
0.535353535353535 0.218423582161079
0.545454545454546 0.218423582161079
0.555555555555556 0.218423582161079
0.565656565656566 0.218423582161079
0.575757575757576 0.218423582161079
0.585858585858586 0.218423582161079
0.595959595959596 0.218423582161079
0.606060606060606 0.218423582161079
0.616161616161616 0.218423582161079
0.626262626262626 0.218423582161079
0.636363636363636 0.218423582161079
0.646464646464647 0.218423582161079
0.656565656565657 0.218423582161079
0.666666666666667 0.218423582161079
0.676767676767677 0.218423582161079
0.686868686868687 0.218423582161079
0.696969696969697 0.218423582161079
0.707070707070707 0.218423582161079
0.717171717171717 0.218423582161079
0.727272727272727 0.21848361999632
0.737373737373737 0.218927027361948
0.747474747474748 0.219659654570136
0.757575757575758 0.220523853828711
0.767676767676768 0.221430091291821
0.777777777777778 0.222397001594682
0.787878787878788 0.223570831061642
0.797979797979798 0.225016610687653
0.808080808080808 0.226698152035704
0.818181818181818 0.228636183551287
0.828282828282828 0.230773884893269
0.838383838383838 0.2334327493306
0.848484848484849 0.236902617915311
0.858585858585859 0.241633360683945
0.868686868686869 0.248401291270278
0.878787878787879 0.257409959808586
0.888888888888889 0.269251659415025
0.898989898989899 0.285505858337935
0.909090909090909 0.307639616146128
0.919191919191919 0.334071185628583
0.929292929292929 0.363677258647668
0.939393939393939 0.39714135378279
0.94949494949495 0.438521561999208
0.95959595959596 0.484408535555417
0.96969696969697 0.536385235530157
0.97979797979798 0.596168713905392
0.98989898989899 0.656219732564771
1 0.716223210232381
};
\addlegendentry{Alg. 1}
\draw[<->,draw=black] (axis cs:0.285714285714286,0.32608797132254) -- (axis cs:0.285714285714286,0.444072559765759);
\draw (axis cs:0.3,0.38508026554415) node[
  scale=0.5,
  anchor=base west,
  text=black,
  rotate=0.0
]{36.15\%};
\draw[<->,draw=black] (axis cs:0.285714285714286,0.218423582161079) -- (axis cs:0.285714285714286,0.32608797132254);
\draw (axis cs:0.3,0.272255776741809) node[
  scale=0.5,
  anchor=base west,
  text=black,
  rotate=0.0
]{49.31\%};
\end{axis}
\end{tikzpicture}
    \caption{Impact of dynamic participation of agent $\{a_4\}$ on normalized payment: Case I, Case II, Alg. 1.}
    \label{fig:dynamic}
\end{figure}

%% file: figmarketutilityref10.tex
\begin{tikzpicture}

\definecolor{darkgray176}{RGB}{176,176,176}
\definecolor{mediumpurple}{RGB}{147,112,219}
\definecolor{orange}{RGB}{255,165,0}
\definecolor{steelblue}{RGB}{70,130,180}

\begin{axis}[
tick align=outside,
tick pos=left,
x grid style={darkgray176},
xlabel={},
xmin=-0.54, xmax=2.54,
xtick style={color=black},
xtick={0,1,2},
xticklabels={Alg. 1,Case II,Case I},
y grid style={darkgray176},
ylabel={Overall market utility},
ymin=-1.31783414486321, ymax=1.34467811974307,
ytick style={color=black}
]
\draw[draw=orange,fill=orange] (axis cs:-0.4,0) rectangle (axis cs:0.4,1.22365483498824);
\draw[draw=steelblue,fill=steelblue] (axis cs:0.6,0) rectangle (axis cs:1.4,-0.678650971791847);
\draw[draw=mediumpurple,fill=mediumpurple] (axis cs:1.6,0) rectangle (axis cs:2.4,-1.19681086010838);
\end{axis}
\end{tikzpicture}

%% file: figmarketutilityref40.tex
\begin{tikzpicture}

\definecolor{darkgray176}{RGB}{176,176,176}
\definecolor{mediumpurple}{RGB}{147,112,219}
\definecolor{orange}{RGB}{255,165,0}
\definecolor{steelblue}{RGB}{70,130,180}

\begin{axis}[ 
tick align=outside,
tick pos=left,
x grid style={darkgray176},
xlabel={},
xmin=-0.54, xmax=2.54,
xtick style={color=black},
xtick={0,1,2},
xticklabels={Alg. 1,Case II,Case I},
y grid style={darkgray176},
ylabel={Overall market utility},
ymin=-0.656948098085382, ymax=0,
ytick style={color=black}
]
\draw[draw=orange,fill=orange] (axis cs:-0.4,0) rectangle (axis cs:0.4,-0.561023854511541);
\draw[draw=steelblue,fill=steelblue] (axis cs:0.6,0) rectangle (axis cs:1.4,-0.598650971791847);
\draw[draw=mediumpurple,fill=mediumpurple] (axis cs:1.6,0) rectangle (axis cs:2.4,-0.625664855319411);
\end{axis}

\end{tikzpicture}

%% file: figconvergence.tex
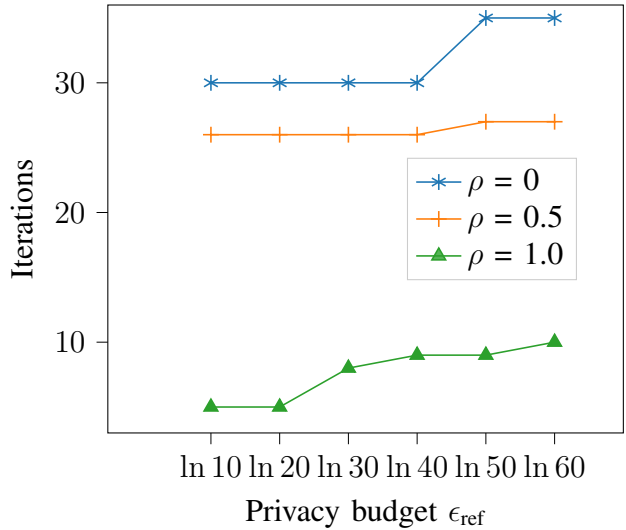
\begin{figure}[t!]
    \centering
 \begin{tikzpicture}

\definecolor{darkgray176}{RGB}{176,176,176}
\definecolor{darkorange25512714}{RGB}{255,127,14}
\definecolor{forestgreen4416044}{RGB}{44,160,44}
\definecolor{lightgray204}{RGB}{204,204,204}
\definecolor{steelblue31119180}{RGB}{31,119,180}

\begin{axis}[
legend cell align={left},
legend style={
  fill opacity=0.8,
  draw opacity=1,
  text opacity=1,
  at={(0.91,0.5)},
  anchor=east,
  draw=lightgray204
},
tick align=outside,
tick pos=left,
x grid style={darkgray176},
xlabel={Privacy budget \(\displaystyle \epsilon_{\textrm{ref}}\) },
xmin=-0.05, xmax=0.7,
xtick style={color=black},
xtick={0.1,0.2,0.3,0.4,0.5,0.6},
xticklabels={$\ln10$,$\ln20$,$\ln30$,$\ln40$,$\ln50$,$\ln60$},
y grid style={darkgray176},
ylabel={Iterations},
ymin=3, ymax=36,
ytick style={color=black}
]
\addplot [semithick, steelblue31119180, mark=asterisk,mark size=3]
table {%
0.1 30
0.2 30
0.3 30
0.4 30
0.5 35
0.6 35
};
\addlegendentry{$\rho$ = 0}
\addplot [semithick, darkorange25512714, mark=+, mark size=3]
table {%
0.1 26
0.2 26
0.3 26
0.4 26
0.5 27
0.6 27
};
\addlegendentry{$\rho$ = 0.5}
\addplot [semithick, forestgreen4416044, mark=triangle*, mark size=3]
table {%
0.1 5
0.2 5
0.3 8
0.4 9
0.5 9
0.6 10
};
\addlegendentry{$\rho$ = 1.0}
\end{axis}

\end{tikzpicture}
    \caption{Convergence analysis (in terms of number of required iterations) at different privacy budgets under data similarity.}
    \label{fig:convergence}
\end{figure}